\documentclass{article}
\usepackage{templates/iclr2027/iclr2027_conference,times}

\usepackage[utf8]{inputenc}
\usepackage[T1]{fontenc}
\usepackage{courier}
\usepackage{hyperref}
\hypersetup{
  hidelinks,
  hypertexnames=false,
  pdftitle={Block Parallelism for Efficient Distributed Long-Context Diffusion Language Model Training},
  pdfauthor={Tarun Suresh, Pranshu Chaturvedi, Hangoo Kang, Parth Shroff, Ishan S. Khare, Hermann Kumbong, Azalia Mirhoseini}
}
\usepackage{url}
\usepackage{booktabs}
\usepackage{amsmath}
\usepackage{amssymb}
\usepackage{amsthm}
\usepackage{algorithm}
\usepackage{algpseudocode}
\usepackage{listings}
\usepackage{float}
\usepackage{microtype}
\usepackage{xcolor}
\usepackage{graphicx}
\usepackage{tikz}
\usetikzlibrary{arrows.meta,calc}

\newtheorem{proposition}{Proposition}
\algrenewcommand\algorithmiccomment[1]{\hfill\textit{(#1)}}

\definecolor{codebg}{RGB}{247,247,247}
\definecolor{codekeyword}{RGB}{196,0,118}
\definecolor{codecomment}{RGB}{25,145,45}
\definecolor{codebuiltin}{RGB}{30,80,180}
\definecolor{codestring}{RGB}{125,45,150}
\definecolor{codenumber}{RGB}{125,125,125}
\lstdefinestyle{implementation}{
  basicstyle=\ttfamily\scriptsize,
  backgroundcolor=\color{codebg},
  keywordstyle=\color{codekeyword}\bfseries,
  commentstyle=\color{codecomment}\itshape,
  stringstyle=\color{codestring},
  emph={cuda_grid_2d,isfinite,exp,log,max,arange,empty_outputs,launch,async_all_gather,async_reduce_scatter,block_diagonal_attention,block_diagonal_backward,view_all_gathered_clean_kv,block_causal_mask,arrange_for_reduce_scatter,masked_clean_attention,masked_clean_backward,masked_attention_backward,pack_local_clean_kv,split_local_clean_gradients,supports_packed_backward,merge_lse_kernel,normalize_output_kernel},
  emphstyle=\color{codebuiltin},
  numbers=left,
  numberstyle=\tiny\color{codenumber},
  numbersep=8pt,
  stepnumber=1,
  columns=fullflexible,
  keepspaces=true,
  showstringspaces=false,
  frame=none,
  xleftmargin=1.8em,
  xrightmargin=0.4em,
  aboveskip=0.6em,
  belowskip=0.4em,
  breaklines=true,
  breakatwhitespace=true,
  captionpos=b
}

\title{\raggedright\fontsize{16}{19}\selectfont
Block Parallelism for Efficient Distributed\\
Long-Context Diffusion Language Model Training}
\author{\bfseries%
Tarun Suresh\thanks{Equal contribution. Correspondence to
\texttt{\{tsuresh@stanford.edu, pranshu@cs.stanford.edu,
hangook@stanford.edu\}}.},\hspace{0.8em}
Pranshu Chaturvedi\footnotemark[1],\hspace{0.8em}
Hangoo Kang\footnotemark[1],\hspace{0.8em}Parth Shroff,\ Ishan S. Khare, \\
\bfseries Hermann Kumbong,\ Azalia Mirhoseini \\[0.5ex]
\normalfont Stanford University
}

\iclrfinalcopy

\newcommand{\preprintfrontierfigure}{%
\begin{figure}[H]
  \centering
  \vspace{-0.4\baselineskip}
  \includegraphics[width=\linewidth]{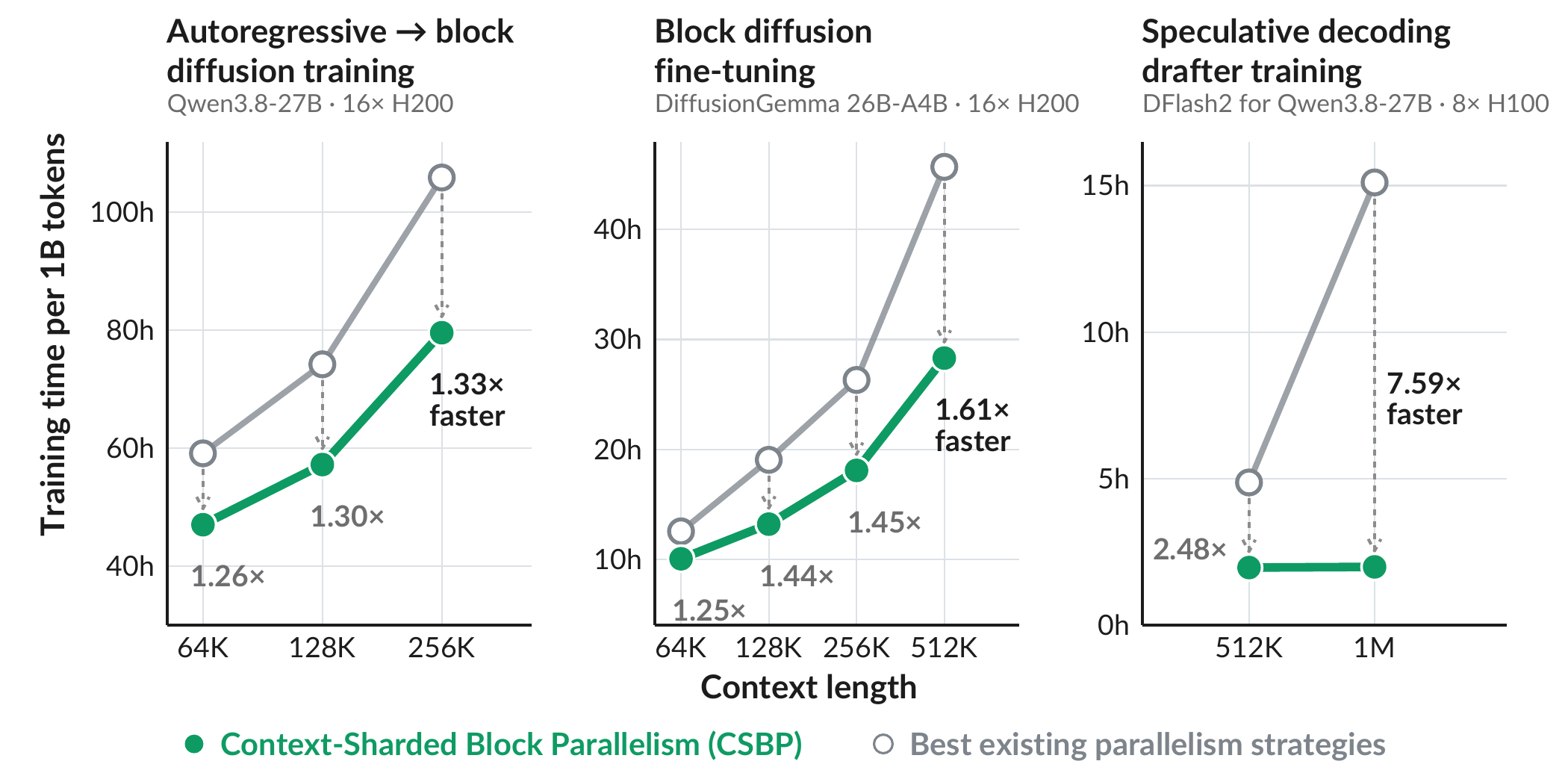}
  \vspace{-1.3\baselineskip}
  \caption{CSBP accelerates the frontier of long-context block diffusion training.}
  \label{fig:csbp-frontier}
\end{figure}%
}

\begin{document}
\maketitle
\lhead{Preprint. Under review.}

\begin{abstract}
Block diffusion language models (BDLMs) combine autoregressive dependencies
across blocks with parallel denoising within blocks, but long-context training
is constrained by distributed attention communication and activation memory.
Conventional context parallelism (CP) shards the combined clean-plus-corrupted
sequence by position, communicating shared clean K/V together with
block-specific corrupted K/V and their gradients. We observe that the BDLM
objective separates over target blocks. We introduce block parallelism (BP), a
new distributed parallelism dimension that assigns each corrupted-block
computation to one rank. To scale BP to long contexts, we introduce
context-sharded block parallelism (CSBP), which also shards the shared clean
sequence across those ranks. CSBP keeps corrupted K/V and
gradients local, avoids replicated clean prefixes, and preserves BDLM training
semantics.
On 16 H200 GPUs at 256K context, CSBP improves
throughput over the best baseline by
\textbf{1.18--1.45$\boldsymbol{\times}$} for supervised fine-tuning and
\textbf{1.27--1.33$\boldsymbol{\times}$} for
conversion of autoregressive models to BDLMs, while matching or reducing peak HBM.
Full-model speedup reaches \textbf{1.61$\boldsymbol{\times}$} at 512K. On eight H100 GPUs,
CSBP accelerates DFlash2 speculative-decoder training by
\textbf{2.48$\boldsymbol{\times}$} at 512K and
\textbf{7.59$\boldsymbol{\times}$} at 1M. In matched 12-hour DiffusionGemma
26B-A4B SFT runs, CSBP achieves higher pass rates at every trained checkpoint
on SWE-bench Verified and Terminal-Bench Lite.
The code is available at \url{https://github.com/ScalingIntelligence/Turbo-dLLM}.
\end{abstract}

\section{Introduction}
\label{sec:introduction}

\ifdefined\preprintfrontierfigure
\preprintfrontierfigure
\fi

Block diffusion language models, or BDLMs, offer a promising approach to
combining strong language-model quality with efficient inference. They retain
autoregressive dependencies across consecutive token blocks while generating
the tokens within each block through parallel denoising. This combines the
causal structure of autoregressive models with the parallelism of
diffusion-based generation
\citep{arriola2025blockdiffusion,odonoghue2026diffusiongemma,fu2026nemotronlabsdiffusion}.

Coding agents and other long-running assistants accumulate interaction
histories~\citep{yang2024sweagent,yao2023react}, motivating long-context BDLM
training.
For these workloads, activation memory and long-context attention are major
bottlenecks. Existing long-context training systems commonly use context
parallelism, which partitions sequence activations across devices so that each
device stores and processes only part of the context. During attention, devices
exchange information across these partitions, allowing local queries to
incorporate relevant information from the full
sequence~\citep{liu2023ringattention,jacobs2023ulysses}.

Conventional context parallelism (CP) is poorly matched to BDLM training
because it partitions the combined clean-plus-corrupted representation by
token position and performs distributed attention across those position
shards. Clean K/V from earlier blocks can condition multiple later blocks,
whereas corrupted K/V for block $b$ are used only by queries from that block.
At every layer, CP therefore exchanges both clean and corrupted K/V during
forward and their corresponding gradients during backward. The corrupted K/V
consume cross-rank bandwidth despite having no reuse across blocks
(Figure~\ref{fig:block-parallel-overview}(a)).

\begin{figure}[t]
    \centering
    \resizebox{\linewidth}{!}{\definecolor{BPBlue}{HTML}{356F9F}
\definecolor{BPOrange}{HTML}{B66718}

\begin{tikzpicture}[
    x=1cm,
    y=1cm,
    font=\sffamily\tiny,
    panel/.style={draw=black!24, rounded corners=2.2pt, line width=0.6pt,
                  minimum width=4.60cm, minimum height=4.35cm},
    clean/.style={draw=BPBlue, fill=BPBlue!13, rounded corners=1.2pt,
                  line width=0.65pt, minimum height=0.48cm,
                  inner xsep=0.05cm, inner ysep=0.03cm, align=center},
    corrupt/.style={draw=BPOrange, fill=BPOrange!15, rounded corners=1.2pt,
                    dash pattern=on 2.3pt off 1.2pt, line width=0.72pt,
                    minimum height=0.48cm, inner xsep=0.05cm,
                    inner ysep=0.03cm, align=center},
    ranklabel/.style={draw=black!42, fill=black!4, rounded corners=1.5pt,
                     line width=0.55pt, minimum width=0.86cm,
                     minimum height=0.48cm, inner xsep=0.05cm,
                     inner ysep=0.03cm, align=center,
                     font=\sffamily\tiny\bfseries},
    rankframe/.style={draw=black!48, fill=black!2, rounded corners=1.8pt,
                     line width=0.62pt, minimum width=1.28cm,
                     minimum height=2.20cm},
    summary/.style={draw=black!32, fill=black!3, rounded corners=1.7pt,
                    line width=0.55pt, minimum width=4.12cm,
                    minimum height=0.82cm, text width=3.88cm,
                    inner xsep=0.08cm, inner ysep=0.04cm, align=center}
]

\node[panel] at (2.40,2.18) {};
\node[panel] at (7.20,2.18) {};
\node[panel] at (12.00,2.18) {};

\node[font=\sffamily\scriptsize\bfseries, text=black!82]
      at (2.40,4.08) {(a) CP: token-wise sharding};
\node[font=\sffamily\scriptsize\bfseries, text=black!82]
      at (7.20,4.08) {(b) BP: block ownership};
\node[font=\sffamily\scriptsize\bfseries, text=black!82, align=center]
      at (12.00,4.02) {(c) CSBP: sharded context\\+ block ownership};

\foreach \x/\r in {0.90/1,2.40/2,3.90/3} {
    \node[rankframe] at (\x,2.18) {};
    \node[font=\sffamily\tiny\bfseries, text=black!76]
          at (\x,3.00) {Rank \r};
    \node[text=BPBlue!88!black] at (\x,2.76) {Clean K/V};
    \node[text=BPOrange!90!black, align=center]
          at (\x,1.99) {Corrupted\\K/V};
    \foreach \i in {1,2,3} {
        \pgfmathtruncatemacro{\tokenpos}{3*(\r-1)+\i}
        \node[draw=BPBlue, fill=BPBlue!13, line width=0.65pt,
              minimum width=0.30cm, minimum height=0.40cm, inner sep=0pt]
              at ({\x+0.34*(\i-2)},2.46) {\tokenpos};
        \node[draw=BPOrange, fill=BPOrange!15, line width=0.65pt,
              minimum width=0.30cm, minimum height=0.40cm, inner sep=0pt]
              at ({\x+0.34*(\i-2)},1.57) {\tokenpos};
    }
}
\foreach \x in {1.47,2.97} {
    \draw[<->, >=stealth, draw=BPBlue, line width=0.8pt]
          (\x,2.46) -- (\x+0.36,2.46);
    \draw[<->, >=stealth, draw=BPOrange, line width=0.8pt]
          (\x,1.57) -- (\x+0.36,1.57);
}
\node[font=\sffamily\tiny\bfseries, text=black!72]
      at (2.40,3.46) {Small cells denote token positions};
\node[summary, font=\sffamily\tiny\bfseries, text=black!72]
      at (2.40,0.50) {Clean and corrupted K/V\\cross ranks during attention};

\node[ranklabel] at (5.76,2.94) {Rank 1};
\node[ranklabel] at (5.76,2.23) {Rank 2};
\node[ranklabel] at (5.76,1.52) {Rank 3};

\node[corrupt, minimum width=0.86cm, font=\sffamily\tiny\bfseries]
      at (6.72,2.94) {Block 1};

\node[clean, minimum width=0.86cm] at (6.72,2.23) {Clean 1};
\node[corrupt, minimum width=0.86cm, font=\sffamily\tiny\bfseries]
      at (7.68,2.23) {Block 2};

\node[clean, minimum width=0.86cm] at (6.72,1.52) {Clean 1};
\node[clean, minimum width=0.86cm] at (7.68,1.52) {Clean 2};
\node[corrupt, minimum width=0.86cm, font=\sffamily\tiny\bfseries]
      at (8.64,1.52) {Block 3};

\node[font=\sffamily\tiny\bfseries, text=BPOrange!90!black]
      at (7.20,3.46) {Each target block has one owner};
\node[font=\sffamily\tiny, text=BPBlue!88!black]
      at (7.20,1.14) {Overlapping clean prefixes are replicated};
\node[summary, draw=BPOrange!60, fill=BPOrange!5,
      font=\sffamily\tiny\bfseries, text=black!72]
      at (7.20,0.50) {Corrupted K/V stay local;\\clean prefixes are replicated};

\foreach \x/\r in {10.50/1,12.00/2,13.50/3} {
    \node[rankframe] at (\x,2.18) {};
    \node[font=\sffamily\tiny\bfseries, text=black!76]
          at (\x,3.00) {Rank \r};
    \node[text=BPBlue!88!black] at (\x,2.76) {Clean K/V};
    \foreach \i in {1,2,3} {
        \pgfmathtruncatemacro{\tokenpos}{3*(\r-1)+\i}
        \node[draw=BPBlue, fill=BPBlue!13, line width=0.65pt,
              minimum width=0.30cm, minimum height=0.40cm, inner sep=0pt]
              at ({\x+0.34*(\i-2)},2.46) {\tokenpos};
    }
    \node[corrupt, minimum width=1.10cm, minimum height=0.68cm,
          text width=0.94cm, inner xsep=0.04cm,
          font=\sffamily\tiny\bfseries] at (\x,1.57)
          {Complete\\block \r};
}

\foreach \x in {11.07,12.57} {
    \draw[<->, >=stealth, draw=BPBlue, line width=0.8pt]
          (\x,2.46) -- (\x+0.36,2.46);
}

\node[font=\sffamily\tiny\bfseries, text=BPBlue!88!black]
      at (12.00,3.46) {Shared clean sequence is token-sharded};
\node[summary, draw=BPOrange!60, fill=BPOrange!5,
      font=\sffamily\tiny\bfseries, text=black!72]
      at (12.00,0.50) {Only shared clean K/V\\cross ranks during attention};
\end{tikzpicture}}
    \caption{(a) Conventional CP shards by token position and exchanges both clean and corrupted K/V across ranks,
    despite corrupted K/V having no reuse across blocks. (b) BP keeps
    corrupted K/V local to each block's owner, but replicates overlapping
    clean prefixes. (c) CSBP also shards the shared clean context, so only
    clean K/V and their gradients cross ranks during attention.
    Figure~\ref{fig:csbp} details the resulting tensor communication.}
    \label{fig:block-parallel-overview}
\end{figure}
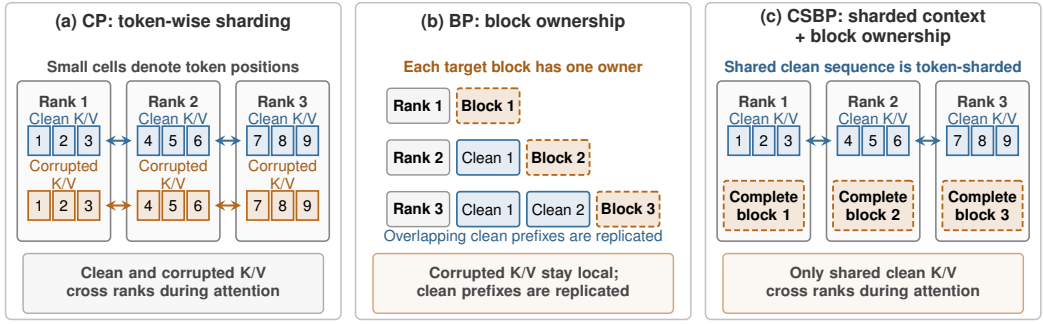

Our first insight is that we can eliminate this block-specific communication by
exploiting the BDLM objective's separability across corrupted blocks. We
introduce block parallelism (BP), which makes target-block computation a new
distributed parallelism dimension. BP assigns each complete corrupted-block
computation to one rank. Different ranks execute different block computations
concurrently while each block's corrupted Q/K/V, activations, and gradients
remain local (Figure~\ref{fig:block-parallel-overview}(b)).

BP alone replicates increasingly long and overlapping clean prefixes across
ranks. To scale BP to long contexts, we introduce context-sharded
block parallelism (CSBP): CP shards the shared clean sequence across the same
ranks that own complete corrupted-block computations. Each clean K/V shard
serves all local blocks whose prefixes include it, and its gradients accumulate
on its owner. CSBP thus distributes clean computation and storage while keeping
corrupted K/V and their gradients local
(Figure~\ref{fig:block-parallel-overview}(c)).

We make four contributions.

\begin{itemize}
    \item \textbf{Block parallelism.} We introduce target-block computation as
    a new distributed parallelism dimension for BDLM training. BP uses
    per-block loss separability to assign each complete corrupted-block
    computation, including its forward pass, loss, and backward pass, to one
    rank. Different blocks execute concurrently while their corrupted K/V and
    gradients remain local, eliminating block-specific cross-rank communication.

    \item \textbf{Context-sharded block parallelism.} We introduce CSBP to
    scale block parallelism to long contexts without replicating the clean
    sequence. Over the same ranks, BP assigns complete corrupted-block
    computations while CP shards the shared clean-sequence computation. Each
    clean K/V shard serves every block whose prefix contains it, and its owner
    accumulates the backward contributions from all such blocks. CSBP removes
    corrupted K/V from cross-rank attention communication and avoids
    replicating overlapping clean prefixes.

    \item \textbf{Efficient long-context training.} Across six models,
    CSBP delivers $256$K-context speedups of
    \textbf{1.18--1.45$\boldsymbol{\times}$} for SFT and
    \textbf{1.27--1.33$\boldsymbol{\times}$} for AR-to-BDLM conversion over the best baselines,
    while matching or reducing peak HBM (Table~\ref{tab:main-256k-results}).
    The advantage grows with CP/BP degree and context length
    (Figures~\ref{fig:parallel-degree-scaling}--\ref{fig:sequence-length-scaling}):
    at 512K, it reaches \textbf{1.19$\boldsymbol{\times}$} on
    NemotronDiffusion 14B and \textbf{1.61$\boldsymbol{\times}$} on
    DiffusionGemma 26B-A4B.

    \item \textbf{Faster speculative decoding training and stronger agent performance.}
    CSBP accelerates DFlash2 drafter training by
    \textbf{2.48$\boldsymbol{\times}$} at 512K and
    \textbf{7.59$\boldsymbol{\times}$} at 1M. Under matched 12-hour
    DiffusionGemma SFT budgets, CSBP achieves higher pass rates at every trained
    checkpoint on SWE-bench Verified and Terminal-Bench Lite.
\end{itemize}

\section{Related Work}
\label{sec:related-work}

\paragraph{Long-context training parallelism.}
Megatron sequence parallelism reduces activation replication alongside
tensor parallelism~\citep{korthikanti2022recomputation}.
DeepSpeed-Ulysses redistributes sequence shards across attention
heads~\citep{jacobs2023ulysses}, while Ring Attention overlaps K/V
communication with attention computation~\citep{liu2023ringattention}.
Striped Attention balances causal attention workloads, and LoongTrain
combines head and context parallelism~\citep{brandon2023striped,gu2024loongtrain}.

\paragraph{Block diffusion and AR-to-BDLM adaptation.}
Block Diffusion combines autoregressive dependencies between blocks with
parallel denoising within blocks, supporting variable-length generation
and clean-prefix K/V caching~\citep{arriola2025blockdiffusion}.
Fast-dLLM~v2 converts AR checkpoints using block-structured attention,
complementary masking, and shifted prediction~\citep{wu2025fastdllmv2}.
DiffusionGemma adapts an AR model through denoising SFT followed by sampler
distillation and reinforcement learning~\citep{diffusiongemma2026report}.

\paragraph{Diffusion-based speculative decoding.}
Speculative decoding verifies draft tokens in parallel while preserving
the AR target distribution~\citep{leviathan2023speculative}.
DFlash uses target-model features to draft an entire block in one
pass~\citep{chen2026dflash}. DFlash2 adds block-local dynamic convolutions
and candidate-path selection to improve draft
quality~\citep{inco2026dflash2}.
CSBP instead accelerates block-separable training by separating shared
context from block-local computation, preserving training objectives and
inference protocols.

\section{Background}
\label{sec:background}

\subsection{Notation}
Let a sequence $x=(x_1,\ldots,x_L)$ be divided into $B$ contiguous blocks of
$M=L/B$ tokens.
\begin{equation}
    x=\left(x^{(1)},\ldots,x^{(B)}\right),
    \qquad
    x^{(<b)}=\left(x^{(1)},\ldots,x^{(b-1)}\right).
\end{equation}
We use $b$ to index blocks and $r$ to index the $P$ ranks in a parallel group.
At each transformer layer~\citep{vaswani2017attention}, $Q$, $K$, and $V$ denote the query, key, and value
projections; semicolons inside brackets denote sequence concatenation.

\subsection{Block diffusion language models}

A block diffusion language model (BDLM) generates blocks from left to right
while denoising the tokens within each block in parallel
\citep{arriola2025blockdiffusion}.
\begin{equation}
    p_\theta(x)
    =
    \prod_{b=1}^{B}
    p_\theta\!\left(x^{(b)}\mid x^{(<b)}\right),
    \label{eq:bdlm-factorization}
\end{equation}
where each conditional distribution is modeled by a diffusion process.

Training samples continuous times $t_b$, not discrete timestep indices:
standard BDLM uses uniform stratification over $[10^{-3},1)$ across batch/block
pairs; Fast-dLLM~v2 uses unstratified $\mathcal{U}[0,1)$ draws.
Each target block forms a corrupted copy $x_{t_b}^{(b)}$ conditioned on its
clean prefix $x^{(<b)}$. The training objective is
\begin{equation}
    \mathcal{L}_{\mathrm{BDLM}}
    =
    \sum_{b=1}^{B}
        \mathbb{E}_{t_b,\,x_{t_b}^{(b)}}\!\left[
            -w(t_b)\log p_\theta\!\left(
                x^{(b)} \mid x_{t_b}^{(b)},x^{(<b)}
            \right)
        \right].
    \label{eq:bdlm-training-objective}
\end{equation}
The expectation is over $t_b$ and
$x_{t_b}^{(b)}\sim q_{t_b}(\cdot\mid x^{(b)})$, where $q_{t_b}$ is the
corruption process and $w(t_b)$ is a noise-dependent loss weight. We denote
the $b$th summand, including token normalization, by $\mathcal{L}_b$. Training
observes the complete sequence and uses the clean ground-truth prefix for every
target block.

The standard BDLM training computation evaluates all $B$ block losses together
by placing a length-$L$ corrupted copy alongside a length-$L$ clean copy,
forming a length-$2L$ representation. It contains a block-specific corrupted
copy and a reusable clean copy of every token. A block-causal attention mask
gives the clean tokens in block $b$ access to the preceding clean blocks and
causal access within block $b$. The corrupted tokens in target block $b$
attend bidirectionally within their own corrupted block and attend to the clean
prefix $x^{(<b)}$, but do not use corrupted K/V from other blocks. Their
visible keys and values are therefore
$[K_b;K_{<b}^{\mathrm{clean}}]$ and
$[V_b;V_{<b}^{\mathrm{clean}}]$.

The corrupted activations of block $b$ affect only $\mathcal{L}_b$, while its
clean representation can be reused by multiple later block losses. Corrupted
computation is therefore block-specific, while clean computation is shared
across block losses. All block losses share the model parameters.

\subsection{Context parallelism}
\label{sec:background-cp}
Context parallelism (CP) uses token position as its ownership dimension,
partitioning the attention input across $P$ ranks so each rank stores
approximately a $1/P$ fraction of its token activations while the group jointly
evaluates attention~\citep{liu2023ringattention,jacobs2023ulysses}. BDLM
training supplies the length-$2L$ clean-plus-corrupted representation as this
attention input, so conventional CP partitions both copies by token position.

Let $Q_r$, $K_r$, and $V_r$ denote the query, key, and value projections
initially owned by rank $r$, and let $d_h$ be the attention-head dimension.
Writing $K=[K_1;\ldots;K_P]$ and $V=[V_1;\ldots;V_P]$, the attention output
for the queries on rank $r$ is
\begin{equation}
O_r
=
\operatorname{softmax}\!\left(
    \frac{Q_rK^\top}{\sqrt{d_h}} + \mathcal{M}_r
\right)V,
\label{eq:context-parallel-attention}
\end{equation}
where $\mathcal{M}_r$ contains the attention-mask rows for the queries on rank
$r$.

Causal attention produces uneven work across sequence shards because later
queries see longer prefixes. Zigzag partitioning pairs early and late chunks
on each rank to balance their attention work~\citep{dubey2024llama3}, while
retaining the sequence shard as the unit of CP ownership.

\section{Block Parallelism}
\label{sec:method}

\subsection{Limitations of conventional context parallelism}
\label{sec:cp-limitations}

Conventional CP is poorly matched to BDLM training because it partitions the
entire length-$2L$ clean-plus-corrupted representation by token position and
performs distributed attention across those position shards. Both clean and
corrupted K/V therefore participate in cross-rank attention communication.
The corrupted K/V for block $b$ are used only by queries from that block. CP
exchanges these block-specific K/V during forward and returns their gradients
to the position-shard owners during backward, adding communication at every
layer for data used by only one block.

Clean K/V are shared across later block losses. When
$\mathcal{L}_b$ uses a clean K/V shard, attention backward produces a gradient
contribution for that shard. This contribution is accumulated on the owning
rank, where backward continues through the clean projections and earlier
layers. Conventional CP therefore combines the required clean-prefix
communication with block-specific corrupted K/V and gradient communication at
every layer.

\subsection{Block-parallel execution}
\label{sec:block-parallel-training}

Our first insight is that we can eliminate this block-specific communication by
exploiting the BDLM objective's per-block loss separability to expose
target-block computation as a new distributed parallelism dimension. The
corrupted activations of block $b$ contribute only to $\mathcal{L}_b$, while
the clean representations of its observed prefix can contribute to multiple
block losses. BP assigns each complete corrupted-block computation to one rank.
The assigned rank executes the corrupted forward pass, loss, and backward pass
for block $b$ while keeping the block's corrupted Q/K/V, activations, and
gradients local. Different ranks execute different block computations
concurrently (Figure~\ref{fig:block-parallel-overview}(b)).

Let $\mathcal{B}_r$ denote the blocks assigned to rank $r$. Rank $r$ computes
\begin{equation}
    \mathcal{L}_r
    =
    \sum_{b\in\mathcal{B}_r}
    \mathcal{L}_b.
    \label{eq:local-block-loss}
\end{equation}
Linearity of differentiation gives the global gradient identity
\begin{equation}
    \sum_{r=1}^{P}\nabla_\theta\mathcal{L}_r
    =
    \sum_{b=1}^{B}\nabla_\theta\mathcal{L}_b
    =
    \nabla_\theta\mathcal{L}_{\mathrm{BDLM}}.
    \label{eq:block-gradient-sum}
\end{equation}

\subsection{Context-sharded block parallelism}
\label{sec:context-sharded-block-parallelism}

BP keeps corrupted-block computations local, but BP alone replicates clean
prefixes across ranks. A rank can reuse the longest prefix required by its
assigned blocks, but these prefixes overlap across ranks, and ranks assigned
later blocks may compute and store nearly the full clean sequence. Clean
computation and activations can therefore be replicated across up to
$\min(B,P)$ active ranks.

Our second insight is that we can scale BP without this replication because BP
and CP operate on complementary parts of the computation over the same $P$
ranks. BP assigns complete corrupted-block computations while CP shards the
clean-sequence computation shared by those blocks. We call this
context-sharded block parallelism (CSBP). Each rank executes a subset of
complete corrupted-block computations and stores one clean-sequence shard
(Figure~\ref{fig:csbp}). Context parallelism supplies the distributed clean
K/V and accumulates their backward contributions, while BP keeps each
corrupted-block computation on its owner.

\begin{figure}[t]
    \centering
    \resizebox{\linewidth}{!}{\begin{tikzpicture}[
    x=1cm,
    y=1cm,
    font=\sffamily\scriptsize,
    >={Latex[length=2.2mm,width=1.5mm]},
    device/.style={draw=black!58, fill=black!2, rounded corners=2.3pt,
                   line width=0.7pt, minimum width=2.64cm,
                   minimum height=2.16cm},
    ownerdevice/.style={draw=BPOrange, fill=BPOrange!4,
                       rounded corners=2.3pt, line width=1.05pt,
                       minimum width=2.64cm, minimum height=2.16cm},
    clean/.style={draw=BPBlue, fill=BPBlue!13, rounded corners=1.4pt,
                  line width=0.7pt, align=center},
    corrupt/.style={draw=BPOrange, fill=BPOrange!16, rounded corners=1.4pt,
                    dash pattern=on 2.4pt off 1.2pt, line width=0.75pt,
                    align=center},
    shard/.style={minimum width=1.10cm, minimum height=0.58cm,
                  text width=1.00cm, inner xsep=0pt, align=center},
    blockjob/.style={minimum width=2.16cm, minimum height=0.78cm,
                     text width=2.06cm, inner xsep=0pt, align=center},
    callout/.style={draw=black!38, fill=black!3, rounded corners=2.3pt,
                    line width=0.65pt, text width=2.90cm,
                    minimum width=3.18cm, minimum height=1.14cm,
                    inner xsep=0.14cm, inner ysep=0.10cm,
                    align=center},
    fwdclean/.style={draw=BPBlue, line width=1.0pt,
                     -{Latex[length=2.4mm,width=1.6mm]}},
    fwdcorrupt/.style={draw=BPOrange, line width=1.0pt,
                       -{Latex[length=2.4mm,width=1.6mm]}},
    bwdclean/.style={draw=BPBlue, line width=0.95pt,
                     dash pattern=on 3pt off 1.6pt,
                     -{Latex[length=2.4mm,width=1.6mm]}},
    bwdcorrupt/.style={draw=BPOrange, line width=0.95pt,
                       dash pattern=on 3pt off 1.6pt,
                       -{Latex[length=2.4mm,width=1.6mm]}},
    groupclean/.style={draw=BPBlue, line width=1.0pt,
                       {Latex[length=2.4mm,width=1.6mm]}-{Latex[length=2.4mm,width=1.6mm]}},
    groupcorrupt/.style={draw=BPOrange, line width=1.0pt,
                         {Latex[length=2.4mm,width=1.6mm]}-{Latex[length=2.4mm,width=1.6mm]}},
    groupgradclean/.style={draw=BPBlue, line width=0.95pt,
                           dash pattern=on 3pt off 1.6pt,
                           {Latex[length=2.4mm,width=1.6mm]}-{Latex[length=2.4mm,width=1.6mm]}},
    groupgradcorrupt/.style={draw=BPOrange, line width=0.95pt,
                             dash pattern=on 3pt off 1.6pt,
                             {Latex[length=2.4mm,width=1.6mm]}-{Latex[length=2.4mm,width=1.6mm]}},
    connector/.style={draw=black!30, line width=0.55pt},
    cleanconnector/.style={draw=BPBlue!65, line width=0.6pt}
]
\definecolor{BPBlue}{HTML}{356F9F}
\definecolor{BPOrange}{HTML}{B66718}

\node[clean, minimum width=0.34cm, minimum height=0.20cm,
      inner sep=0pt] at (4.02,6.02) {};
\node[anchor=west, text=black!72, font=\sffamily\tiny]
      at (4.27,6.02) {Clean tokens};
\node[draw=BPOrange, fill=BPOrange!15, minimum width=0.34cm, minimum height=0.20cm,
      inner sep=0pt] at (6.57,6.02) {};
\node[anchor=west, text=black!72, font=\sffamily\tiny]
      at (6.82,6.02) {Corrupted tokens};

\node[anchor=east, font=\sffamily\scriptsize\bfseries, text=black!82]
      at (1.85,4.28) {(a) CP};

\foreach \x/\r/\n in {3.18/{1}/1,6.18/{2}/2,9.18/{3}/3} {
    \node[device] (cpd\n) at (\x,4.28) {};
    \node[font=\sffamily\scriptsize\bfseries, text=black!76]
          at ([yshift=-0.33cm]cpd\n.north) {Rank \r};
    \node[font=\sffamily\tiny, align=center, text=BPBlue!88!black]
          at ({\x-0.57},4.62) {Clean\\$Q/K/V$};
    \node[font=\sffamily\tiny, align=center, text=BPOrange!90!black]
          at ({\x+0.57},4.62) {Corrupted\\$Q/K/V$};
    \foreach \i in {1,2,3} {
        \pgfmathtruncatemacro{\tokenpos}{3*(\n-1)+\i}
        \node[draw=BPBlue, fill=BPBlue!13, line width=0.65pt,
              minimum width=0.30cm, minimum height=0.40cm,
              inner sep=0pt, font=\sffamily\tiny]
              at ({\x-0.57+0.34*(\i-2)},4.16) {\tokenpos};
        \node[draw=BPOrange, fill=BPOrange!15, line width=0.65pt,
              minimum width=0.30cm, minimum height=0.40cm,
              inner sep=0pt, font=\sffamily\tiny]
              at ({\x+0.57+0.34*(\i-2)},4.16) {\tokenpos};
    }
    \node[font=\sffamily\tiny\bfseries, text=black!72, align=center]
          at ([yshift=0.48cm]cpd\n.south) {$Q$, $dQ$, and $O$\\[-1pt]stay local};
}

\node[anchor=east, font=\sffamily\tiny\bfseries, text=black!68]
      at (1.93,5.61) {Forward: $K/V$};
\draw[groupclean] (2.13,5.74) -- (10.23,5.74);
\draw[groupcorrupt] (2.13,5.48) -- (10.23,5.48);
\foreach \x in {3.18,6.18,9.18} {
    \draw[connector] (\x,5.36) -- (\x,5.74);
}

\node[anchor=east, font=\sffamily\tiny\bfseries, text=black!68]
      at (1.93,2.89) {Backward: $dK/dV$};
\draw[groupgradclean] (2.13,3.02) -- (10.23,3.02);
\draw[groupgradcorrupt] (2.13,2.76) -- (10.23,2.76);
\foreach \x in {3.18,6.18,9.18} {
    \draw[connector] (\x,3.20) -- (\x,2.76);
}

\node[callout, font=\sffamily\tiny] at (12.34,4.28)
      {$Q$, $dQ$, $O$ stay local\\[-1pt]
       \textbf{Cross-rank: Clean + Corrupted}\\[-1pt]
       \textbf{$K/V$ and $dK/dV$}};

\node[anchor=east, font=\sffamily\scriptsize\bfseries, text=black!82]
      at (1.85,1.28) {(b) CSBP};

\node[device, minimum height=2.26cm] (fd1) at (3.18,1.23) {};
\node[device, minimum height=2.26cm] (fd2) at (6.18,1.23) {};
\node[ownerdevice, minimum height=2.26cm] (fd3) at (9.18,1.23) {};
\foreach \r/\n in {1/1,2/2,3/3} {
    \node[font=\sffamily\scriptsize\bfseries, text=black!76]
          at ([yshift=-0.33cm]fd\n.north) {Rank \r};
    \node[font=\sffamily\tiny, text=BPBlue!88!black]
          at ([yshift=0.53cm]fd\n.center) {Clean $Q/K/V$};
    \foreach \i in {1,2,3} {
        \pgfmathtruncatemacro{\tokenpos}{3*(\n-1)+\i}
        \pgfmathsetmacro{\tokenshift}{0.34*(\i-2)}
        \node[draw=BPBlue, fill=BPBlue!13, line width=0.65pt,
              minimum width=0.30cm, minimum height=0.40cm,
              inner sep=0pt, font=\sffamily\tiny]
              at ([xshift=\tokenshift cm,yshift=0.20cm]fd\n.center)
              {\tokenpos};
    }
}
\node[corrupt, draw=BPOrange, fill=BPOrange!15, text=black!85,
      blockjob, font=\sffamily\tiny] at ([yshift=-0.59cm]fd1.center)
      {Block $1$: Corrupted\\$Q/K/V + dQ/dK/dV$ local};
\node[corrupt, draw=BPOrange, fill=BPOrange!15, text=black!85,
      blockjob, font=\sffamily\tiny] at ([yshift=-0.59cm]fd2.center)
      {Block $2$: Corrupted\\$Q/K/V + dQ/dK/dV$ local};
\node[corrupt, blockjob, font=\sffamily\tiny\bfseries]
      at ([yshift=-0.59cm]fd3.center)
      {Block $3$: Corrupted\\$Q/K/V + dQ/dK/dV$ local};

\node[anchor=east, font=\sffamily\tiny\bfseries, text=black!68]
      at (1.93,2.56) {Forward: $K/V$};
\draw[fwdclean] (3.18,2.64) -- (9.18,2.64);
\draw[fwdclean] (6.18,2.48) -- (9.18,2.48);
\draw[cleanconnector] (fd1.north) -- (3.18,2.64);
\draw[cleanconnector] (fd2.north) -- (6.18,2.64);
\draw[cleanconnector] (fd3.north) -- (9.18,2.64);

\node[anchor=east, font=\sffamily\tiny\bfseries, text=black!68]
      at (1.93,-0.08) {Backward: $dK/dV$};
\draw[bwdclean] (9.18,-0.13) -- (3.18,-0.13);
\draw[bwdclean] (9.18,0.02) -- (6.18,0.02);
\draw[cleanconnector] (fd1.south) -- (3.18,-0.13);
\draw[cleanconnector] (fd2.south) -- (6.18,0.02);
\draw[cleanconnector] (fd3.south) -- (9.18,-0.13);

\node[callout, draw=BPOrange!78, fill=BPOrange!7,
      font=\sffamily\tiny] at (12.34,1.28)
      {$Q$, $dQ$, $O$ stay local\\[-1pt]
       \textbf{Cross-rank: Clean $K/V$}\\[-1pt]
       \textbf{and Clean $dK/dV$ only}};
\end{tikzpicture}}
    \caption{Conventional CP shards the combined clean-plus-corrupted sequence
    by token position. Clean and corrupted $K/V$ cross ranks during forward,
    and clean and corrupted $dK/dV$ cross ranks during backward; $Q$, $dQ$, and
    attention outputs remain with their query owners. CSBP shards only the
    shared clean sequence, so only clean $K/V$ and $dK/dV$ cross ranks. Each
    block owner retains its corrupted $Q/K/V$ and $dQ/dK/dV$ locally.}
    \label{fig:csbp}
\end{figure}
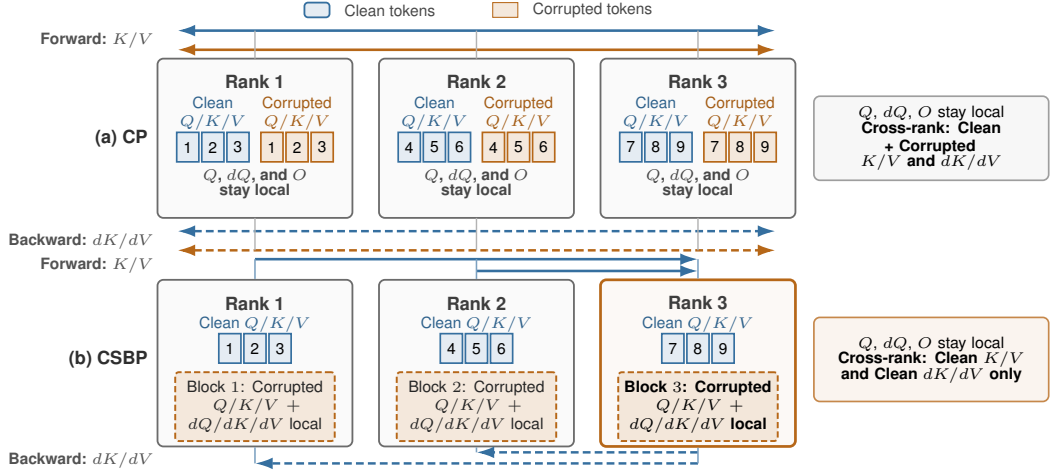

\paragraph{Forward pass.}
The rank assigned target block $b$ computes
\begin{equation}
    O_b
    =
    \operatorname{Attn}\!\left(
        Q_b,
        [K_b;K_{<b}^{\mathrm{clean}}],
        [V_b;V_{<b}^{\mathrm{clean}}]
    \right),
    \label{eq:csbp-block-attention}
\end{equation}
where $\operatorname{Attn}$ denotes scaled dot-product attention. The rank
holds $Q_b$, $K_b$, and $V_b$ locally, while CP supplies the sharded
clean-prefix K/V. Online softmax combines the local within-block contribution
with the distributed clean-prefix
contributions~\citep{milakov2018online,dao2022flashattention}.

The target-block computations reuse overlapping prefixes of the clean
sequence. During one context-parallel attention pass, rank $r$ applies each
clean K/V shard supplied by CP to its clean queries and to every
$b\in\mathcal{B}_r$ whose prefix contains tokens from that shard. The
length-$L$ distributed clean K/V stream therefore serves every local block
whose clean prefix includes that shard. Algorithm~\ref{alg:csbp-attention}
gives the complete per-layer computation in the same notation.

\paragraph{Backward pass.}
The rank assigned block $b$ backpropagates $\mathcal{L}_b$, computing the
gradients of $Q_b$, $K_b$, and $V_b$ locally together with contributions to
the clean-prefix K/V gradients. Exact CP backward accumulates the clean K/V
gradient contributions from all clean and assigned corrupted-block queries on
each clean-shard owner, which continues backward through the clean projections
and earlier layers. Corrupted Q/K/V gradients therefore remain on block owners,
while gradients through clean activations are accumulated on clean-shard
owners. Together, these contributions give the global parameter gradient in
Equation~\ref{eq:block-gradient-sum}. Algorithm~\ref{alg:csbp-backward}
gives the corresponding backward pass. Implementations~\ref{impl:csbp-forward}
and~\ref{impl:csbp-backward} give the production forward and backward fast
paths.

Later target blocks and clean-sequence chunks have longer prefixes and require
more attention work. We devise a dual-end strategy that assigns target blocks
$b$ and $B+1-b$ to the same rank. We use the same early--late pairing for
target blocks and zigzag clean-sequence sharding over the same
ranks~\citep{dubey2024llama3}. Appendix~\ref{app:load-balancing} shows that this
common pairing exactly balances target-block and clean causal attention for
equal partitions and handles non-divisible block counts.

\subsection{Complexity and sources of improvement}
\label{sec:complexity-analysis}

CSBP preserves the exact BDLM computation and gradients
(Proposition~\ref{prop:csbp-exactness}) while eliminating cross-rank
communication for corrupted K/V and avoiding replicated clean prefixes.

\paragraph{Computation.}
Conventional CP and CSBP evaluate the same valid attention pairs, so
both require $\Theta(L^2)$ total attention work and $\Theta(L^2/P)$ work per
rank under balanced assignment. CSBP changes where this work is
executed. Each corrupted block is evaluated on one rank, while CP supplies the
shared clean K/V. The gain comes from keeping block-specific computation local
and eliminating its cross-rank communication while preserving the same
asymptotic attention work.

\paragraph{Communication.}
Conventional CP applies distributed attention to the combined length-$2L$
clean-plus-corrupted representation. CSBP distributes only the
length-$L$ clean K/V, while corrupted K/V remain on their block owners. Each
communicated clean K/V shard serves every local block query whose prefix
includes that shard. In backward, corrupted K/V gradients remain on block
owners and clean K/V gradients return to their shard owners. Cross-rank
attention communication therefore contains only shared clean K/V and their
gradients. CSBP removes the length-$L$ corrupted component from this
communication in both forward and backward.

\paragraph{Memory.}
Each CSBP rank stores one length-$L/P$ clean shard and the activations for its
assigned corrupted blocks. With balanced block assignments, each rank stores
approximately $L/P$ clean tokens and $L/P$ corrupted tokens, giving
approximately $2L/P$ token activations and matching conventional CP to first
order. Pure BP can store a clean prefix approaching length $L$ on ranks
assigned later blocks. CSBP reduces this clean-prefix component to
$L/P$, providing up to a $P\times$ reduction for a rank whose longest required
prefix approaches $L$. CSBP therefore retains CP's activation scaling
without pure BP's overlapping-prefix replication.

\section{Experiments}
\label{sec:experiments}

\subsection{Experimental setup}
\label{sec:experimental-setup}

\paragraph{\mbox{Models and workloads.}}
We evaluate three workloads. For supervised fine-tuning (SFT), we use
NemotronDiffusion 3B, 8B, and 14B~\citep{fu2026nemotronlabsdiffusion} and
DiffusionGemma 26B-A4B~\citep{odonoghue2026diffusiongemma}. For AR-to-BDLM
conversion, we use the text backbones of Qwen3.5-27B~\citep{qwen2026qwen35}
and Qwen3.8-27B~\citep{qwen2026qwen38} with the Fast-dLLM v2
objective~\citep{wu2025fastdllmv2}. Conversion provides a route to parallel
generation without training a diffusion model from scratch and tests CSBP
beyond fine-tuning existing BDLMs. We also train long-context DFlash2
drafters~\citep{inco2026dflash2} for Qwen3.8-27B and Muse-Glimmer-30B following
the SpecForge framework~\citep{li2026specforge} and open-source
implementation~\citep{specforge2025} (Appendix~\ref{app:dflash2-objective}).
Across all workloads, the best baseline and CSBP evaluate the same objective on the same
target blocks. The SFT and AR-to-BDLM workloads include every eligible target
block, while DFlash2 uses the draft blocks sampled by its training objective.

\paragraph{\mbox{Hardware.}}
Full-model training uses two nodes, each with eight NVIDIA H200 GPUs and
141\,GB of HBM3e per GPU. DFlash2 training uses one node with eight NVIDIA
H100 GPUs and 80\,GB of HBM3 per GPU.

\paragraph{\mbox{Baselines and configurations.}}
We compare against highly optimized conventional training without BP, using
the token-position sharding found in distributed-training
libraries~\citep{nvidia_megatron_cp,pytorch_torchtitan} when CP is required. We
profile-optimize baseline execution and share applicable kernel and layout
optimizations with CSBP. For each setting, the best baseline is the
highest-throughput tested feasible combination of data,
tensor~\citep{shoeybi2019megatron}, expert~\citep{lepikhin2020gshard},
context~\citep{liu2023ringattention,jacobs2023ulysses}, and sequence
parallelism~\citep{korthikanti2022recomputation} and global batch size. With
CSBP, each complete corrupted-block computation stays on one rank while the
reusable clean context remains sharded across ranks, eliminating cross-rank
transport of corrupted K/V and their gradients. We optimize topology and batch
size independently with CSBP enabled.

\paragraph{\mbox{Implementation and metrics.}}
We use PyTorch~\citep{paszke2019pytorch},
FlashAttention-4~\citep{zadouri2026flashattention4},
FlexAttention~\citep{dong2025flexattention}, and NCCL; all runs use BF16.
We implement a token-chunked linear cross-entropy loss with cuBLAS
projections and fused CUDA softmax/CE kernels, avoiding full-sequence logits
materialization.
Full-model runs use activation checkpointing~\citep{chen2016checkpointing} and
DeepSpeed ZeRO Stage~2~\citep{rajbhandari2019zero}; the
throughput-maximizing DFlash2 recipe uses fused AdamW without activation
checkpointing. We report useful model FLOPs utilization
(MFU)~\citep{chowdhery2022palm}. For DFlash2, throughput and useful MFU count
supervised draft-target tokens; computation replicated across context ranks
consumes runtime without increasing either metric. Peak HBM is the maximum
peak allocated memory across GPUs. For throughput and memory profiles, every training sequence has
exactly the reported context length. The equal-wall-clock experiments use the
variable sequence lengths of their respective datasets.

\subsection{Effectiveness of CSBP}
\label{sec:main-256k-results}

CSBP improves throughput across every evaluated model and both training
objectives at 256K context (Table~\ref{tab:main-256k-results}). It delivers
$1.18$--$1.45\times$ speedup for SFT and $1.27$--$1.33\times$ for AR-to-BDLM conversion, while matching or reducing peak HBM. The gains extend from
NemotronDiffusion to DiffusionGemma 26B-A4B and both Qwen backbones, demonstrating consistent
benefits across the evaluated model families.

\begin{table}[!htb]
  \centering
  \caption{CSBP accelerates training at 256K context.}
  \label{tab:main-256k-results}
  \small
  \setlength{\tabcolsep}{4pt}
  \renewcommand{\arraystretch}{1.35}
  \resizebox{\linewidth}{!}{%
  \begin{tabular}{@{}llllrrrr@{}}
    \toprule
    \shortstack{Training\\workload} & Model & \shortstack{Best baseline\\topology} & \shortstack{Ours\\topology} & Tok/s & Speedup & \shortstack{MFU\\(\%)} & \shortstack{Peak HBM\\(GiB)} \\
    \midrule
    BDLM fine-tuning & NemotronDiffusion 3B & DP4/TP1/CP4 & DP4/TP1/CP4/BP4 & 13{,}434 / \textbf{16{,}027} & \textbf{1.19$\times$} & 31.7 / \textbf{37.8} & 57.1 / \textbf{57.0} \\
    BDLM fine-tuning & NemotronDiffusion 8B & DP4/TP1/CP4 & DP4/TP1/CP4/BP4 & 9{,}668 / \textbf{11{,}379} & \textbf{1.18$\times$} & 32.3 / \textbf{38.0} & 97.5 / 97.5 \\
    BDLM fine-tuning & NemotronDiffusion 14B & DP2/TP1/CP8 & DP2/TP1/CP8/BP8 & 7{,}884 / \textbf{9{,}301} & \textbf{1.18$\times$} & 33.1 / \textbf{39.0} & 91.1 / \textbf{90.6} \\
    BDLM fine-tuning & DiffusionGemma 26B-A4B & DP1/TP1/EP2/CP8 & DP1/TP1/EP2/CP8/BP8 & 10{,}557 / \textbf{15{,}359} & \textbf{1.45$\times$} & 12.7 / \textbf{18.4} & 118.1 / \textbf{113.6} \\
    \midrule
    AR-to-BDLM conversion & Qwen3.8-27B & DP1/TP2/CP8 & DP1/TP2/CP8/BP8 & 2{,}623 / \textbf{3{,}490} &
    \textbf{1.33$\times$} & 20.1 / \textbf{26.8} & 120.7 / \textbf{111.1} \\
    AR-to-BDLM conversion & Qwen3.5-27B & DP1/TP2/CP8 & DP1/TP2/CP8/BP8 & 2{,}641 / \textbf{3{,}351} &
    \textbf{1.27$\times$} & 20.3 / \textbf{25.7} & 120.7 / \textbf{111.1} \\
    \bottomrule
  \end{tabular}}
  \par\vspace{2pt}
  \parbox{\linewidth}{\footnotesize\raggedright Pairs: best baseline~/~ours. Bold marks the better value.}
\end{table}

\subsection{Scaling with CP/BP degree}
\label{sec:parallel-degree-scaling}

CSBP's advantage grows as training uses more ranks to accommodate longer
contexts (Figure~\ref{fig:parallel-degree-scaling}). From degree two to eight,
speedup rises from $1.10\times$ to $1.18\times$ for NemotronDiffusion 14B and from
$1.25\times$ to $1.45\times$ for DiffusionGemma 26B-A4B.
Keeping corrupted-block computation local becomes increasingly beneficial
in this long-context regime.

\begin{figure}[!htb]
  \centering
  \includegraphics[width=0.92\linewidth]{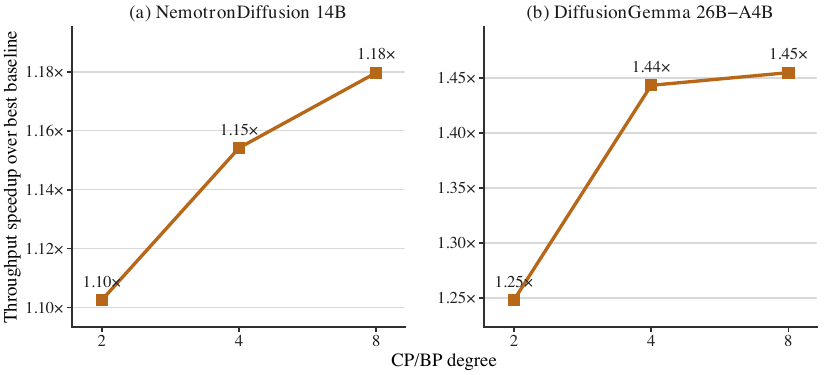}
  \caption{CSBP's throughput gains grow with CP/BP degree.}
  \label{fig:parallel-degree-scaling}
\end{figure}

\subsection{Scaling with context length}
\label{sec:sequence-length-scaling}

Longer contexts strengthen the throughput advantage of CSBP
(Figure~\ref{fig:sequence-length-scaling}). Extending context from 64K to 512K
increases speedup from $1.10\times$ to $1.19\times$ for NemotronDiffusion 14B and from
$1.25\times$ to $1.61\times$ for DiffusionGemma 26B-A4B. The gains persist through
512K, consistent with avoiding corrupted K/V communication becoming more
valuable as sequences grow.

\begin{figure}[!htb]
  \centering
  \includegraphics[width=0.8\linewidth]{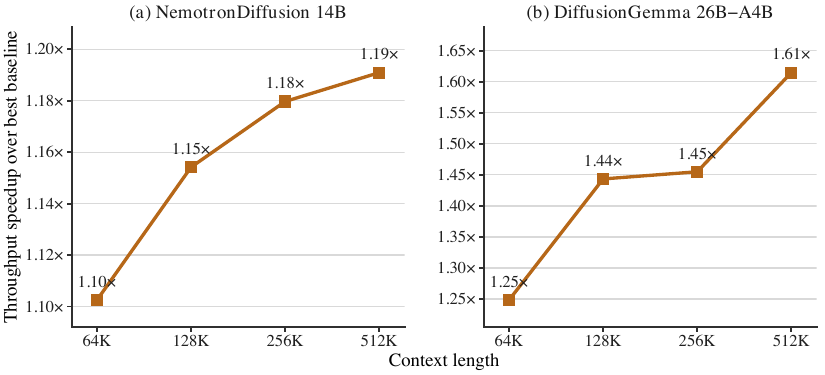}
  \caption{Longer contexts amplify CSBP's throughput advantage.}
  \label{fig:sequence-length-scaling}
\end{figure}

Context sharding is what makes these long-context gains practical.
Without it, BP replicates overlapping clean prefixes: NemotronDiffusion 14B pure BP
is slower and uses more memory at 64K, then runs out of memory at 128K;
DiffusionGemma 26B-A4B pure BP runs out of memory at both 64K and 128K under
the matched native objective.
CSBP fits and accelerates training in these cases
(Appendix~\ref{app:pure-bp-ablation}).

\subsection{Accelerating block diffusion speculative decoding drafter training}
\label{sec:dflash2-scaling}

On eight H100 GPUs (Table~\ref{tab:dflash2-h100-scaling}), CSBP produces larger
DFlash2 speedups than full-model workloads because conventional CP replicates
the block-local decoder, vocabulary loss, and selector across context ranks,
whereas CSBP assigns each draft block to one rank. At 512K, removing this two-way duplication
yields $2.48\times$ higher throughput; at 1M, four CP2/BP2 replicas versus one
CP8 baseline yield $7.59\times$. Executed-GEMM utilization (executed GEMM
FLOP/s relative to BF16 peak) is 28.7\%/37.0\% at 512K and 35.8\%/36.5\% at 1M
for baseline/CSBP. At 512K, CSBP cuts exposed attention communication from
105.4 to 7.5\,ms per GPU because each block stays on one rank, avoiding
corrupted K/V and gradient transport and explaining its higher utilization;
the near-equal 1M values show the gains reflect genuine reductions in executed
work and increased data parallelism. At 256K, where full replicas fit, DP8 remains
19.2\% faster because it avoids sequence-parallel overhead.

Muse-Glimmer-30B shows complementary gains. Against the fastest
feasible non-BP topology, CSBP is $1.56\times$, $3.35\times$, and $4.81\times$
faster at 256K, 512K, and 1M, respectively. At 512K, CP2 runs out of memory,
while CP2/BP2 fits and supports four data-parallel replicas; at 1M, CP4/BP4
both saves 12.0\% peak HBM relative to CP8 and raises throughput by
$4.81\times$.

\begin{table}[!htb]
  \centering
  \caption{CSBP accelerates long-context DFlash2 drafter training.}
  \label{tab:dflash2-h100-scaling}
  \small
  \setlength{\tabcolsep}{4pt}
  \renewcommand{\arraystretch}{1.35}
  \resizebox{\linewidth}{!}{%
  \begin{tabular}{@{}llllrrrr@{}}
    \toprule
    Model & \shortstack{Context\\length} & \shortstack{Best baseline\\topology} & \shortstack{Ours\\topology} & Tok/s & Speedup & \shortstack{MFU\\(\%)} & \shortstack{Peak HBM\\(GiB)} \\
    \midrule
    Qwen3.8-27B & 256K & DP8 & DP4/CP2/BP2
      & \textbf{167{,}228} / 140{,}350 & 0.84$\times$
      & \textbf{38.4} / 32.2
      & 63.5 / \textbf{48.3} \\
    Qwen3.8-27B & 512K & DP4/CP2 & DP4/CP2/BP2
      & 57{,}027 / \textbf{141{,}222} & \textbf{2.48$\times$}
      & 13.1 / \textbf{32.4}
      & 63.3 / \textbf{54.6} \\
    Qwen3.8-27B & 1M & DP1/CP8 & DP4/CP2/BP2
      & 18{,}389 / \textbf{139{,}541} & \textbf{7.59$\times$}
      & 4.2 / \textbf{32.0}
      & \textbf{53.7} / 67.1 \\
    \midrule
    Muse-Glimmer-30B & 256K & DP4/CP2 & DP4/CP2/BP2
      & 62{,}698 / \textbf{98{,}046} & \textbf{1.56$\times$}
      & 18.5 / \textbf{28.9}
      & 63.8 / \textbf{57.5} \\
    Muse-Glimmer-30B & 512K & DP2/CP4 & DP4/CP2/BP2
      & 29{,}493 / \textbf{98{,}785} & \textbf{3.35$\times$}
      & 8.7 / \textbf{29.1}
      & \textbf{63.8} / 65.6 \\
    Muse-Glimmer-30B & 1M & DP1/CP8 & DP2/CP4/BP4
      & 12{,}842 / \textbf{61{,}742} & \textbf{4.81$\times$}
      & 3.8 / \textbf{18.2}
      & 63.3 / \textbf{55.7} \\
    \bottomrule
  \end{tabular}}
  \par\vspace{2pt}
  \parbox{\linewidth}{\footnotesize\raggedright Pairs: best baseline~/~ours. Tok/s and MFU count supervised draft-target tokens. Bold marks the better value.}
\end{table}

\subsection{Equal-wall-clock downstream performance}
\label{sec:equal-wall-clock}

We finally test whether CSBP's higher training throughput translates into
faster downstream progress under a fixed wall-clock budget. We fine-tune
DiffusionGemma 26B-A4B with the best baseline and CSBP for 12 hours and
evaluate checkpoints every three hours on SWE-bench Verified and Terminal-Bench Lite
(Figure~\ref{fig:equal-wall-clock-downstream}). We train with LoRA adapters on
eight H100 GPUs, comparing CP8 with CSBP at CP8/BP8. Appendix~\ref{app:training-evaluation-details}
provides full data and training details. CSBP achieves a higher pass
rate at every trained checkpoint on both benchmarks. Its lead peaks at 1.8
percentage points on SWE-bench Verified and 2 points on Terminal-Bench Lite;
after 12 hours, it remains 1 point ahead on both benchmarks.

\begin{figure}[!htb]
  \centering
  \includegraphics[width=0.95\linewidth]{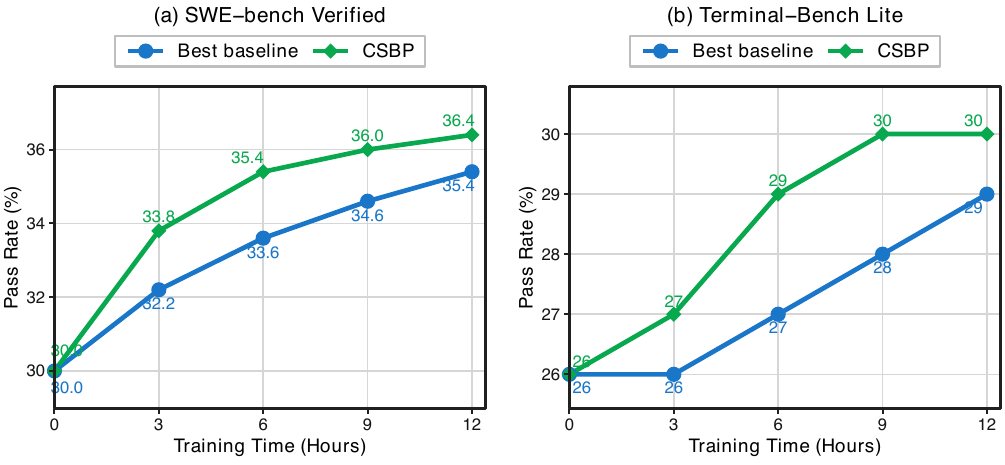}
  \caption{Equal-wall-clock downstream performance. Checkpoints from the
  best baseline and CSBP are evaluated every three hours on
  (a) SWE-bench Verified and
  (b) Terminal-Bench Lite.}
  \label{fig:equal-wall-clock-downstream}
\end{figure}

\section{Limitations}
\label{sec:limitations}
CSBP requires known training targets and accelerates training, not generation
of unknown future tokens. It targets long contexts, where attention
communication and clean-prefix replication are most costly.

\section{Conclusion}
\label{sec:conclusion}
In this work, we propose block parallelism (BP), a distributed parallelism
dimension that assigns complete corrupted-block computations to ranks.
We further propose context-sharded block parallelism (CSBP), which shards
the shared clean context across those ranks to scale training to long
contexts. By matching the parallel decomposition to the BDLM objective,
we keep corrupted K/V and gradients local, avoid clean-prefix replication,
and preserve the reference computation and gradients.

On 16 H200 GPUs at 256K context, we obtain \textbf{1.18--1.45$\boldsymbol{\times}$} higher
throughput for SFT and \textbf{1.27--1.33$\boldsymbol{\times}$} for AR-to-BDLM conversion over
the best baseline, while matching or reducing peak HBM; full-model speedup
reaches \textbf{1.61$\boldsymbol{\times}$} at 512K. On eight H100 GPUs, we accelerate DFlash2
drafter training by \textbf{2.48$\boldsymbol{\times}$} at 512K and \textbf{7.59$\boldsymbol{\times}$} at 1M.
In matched 12-hour DiffusionGemma 26B-A4B SFT runs, we achieve higher pass
rates at every trained checkpoint on SWE-bench Verified and Terminal-Bench
Lite, finishing one percentage point ahead on both benchmarks.

We show that block-aware parallelism improves long-context BDLM training
efficiency and post-training outcomes within fixed budgets without changing
the objective. BP and CSBP complement data, tensor, and expert parallelism
for scaling BDLM adaptation, AR-to-BDLM conversion, and speculative-drafter
training.

\label{iclr:main-text-end}

\clearpage

\subsection*{Reproducibility statement}
Section~\ref{sec:method} describes BP and CSBP, and
Appendix~\ref{app:correctness} states the assumptions and establishes
equivalence to the reference computation.
Appendix~\ref{app:execution-schedule} gives forward and backward schedules.
Section~\ref{sec:experimental-setup} and
Appendix~\ref{app:training-evaluation-details} describe training and evaluation;
Appendices~\ref{app:fast-dllm-v2-objective} and~\ref{app:dflash2-objective}
detail the conversion and drafter objectives.

\subsection*{Acknowledgments}
We thank the Scaling Intelligence Lab and others for their constructive
feedback during the composition of the paper, including Shayan Talaei, Jon
Saad-Falcon, Seil Kang, Ravenor Davion, and Amir Zeinali. We gratefully
acknowledge support from federal sources: NSF under No.~24-554 (AIMing) and
DARPA under No.~HR00112520038 (Fallingwater). We also gratefully acknowledge
support from Stanford HAI, IBM (Code Generation), Lightspeed, Google, and
Google DeepMind. This work was also partially supported by IBM and Felicis as
an affiliate member of Stanford Institute for Human-centered Artificial
Intelligence (HAI)'s industry program. We also gratefully acknowledge computing
resources provided by Stanford's Marlowe Computing
Instrument~\citep{kapfer2025marlowe}.

\bibliographystyle{templates/iclr2027/iclr2027_conference}
\bibliography{references}

\begin{thebibliography}{41}
\providecommand{\natexlab}[1]{#1}
\providecommand{\url}[1]{\texttt{#1}}
\expandafter\ifx\csname urlstyle\endcsname\relax
  \providecommand{\doi}[1]{doi: #1}\else
  \providecommand{\doi}{doi: \begingroup \urlstyle{rm}\Url}\fi

\bibitem[Ariyak et~al.(2026)Ariyak, Zhang, Wang, Zhu, Bianchi, Srivastava,
  Panda, Bharti, Xu, Heo, Wu, Zou, Liang, Song, Zhang, Athiwaratkun, Zhou, and
  Wu]{CoderForge2026}
Alpay Ariyak, Junda Zhang, Junxiong Wang, Shang Zhu, Federico Bianchi, Sanjana
  Srivastava, Ashwinee Panda, Siddhant Bharti, Chenfeng Xu, John Heo,
  Xiaoxia~Shirley Wu, James Zou, Percy Liang, Leon Song, Ce~Zhang, Ben
  Athiwaratkun, Zhongzhu Zhou, and Qingyang Wu.
\newblock {CoderForge-Preview}: {SOTA} open dataset for training efficient
  agents, February 2026.
\newblock URL \url{https://www.together.ai/blog/coderforge-preview}.
\newblock Project core leads: Alpay Ariyak; Zhongzhu Zhou; Qingyang Wu.

\bibitem[Arriola et~al.(2025)Arriola, Gokaslan, Chiu, Yang, Qi, Han, Sahoo, and
  Kuleshov]{arriola2025blockdiffusion}
Marianne Arriola, Aaron~Kerem Gokaslan, Justin~T. Chiu, Zhihan Yang, Zhixuan
  Qi, Jiaqi Han, Subham~Sekhar Sahoo, and Volodymyr Kuleshov.
\newblock Block diffusion: Interpolating between autoregressive and diffusion
  language models.
\newblock In \emph{International Conference on Learning Representations}, 2025.
\newblock URL \url{https://arxiv.org/abs/2503.09573}.

\bibitem[Brandon et~al.(2023)Brandon, Nrusimha, Qian, Ankner, Jin, Song, and
  Ragan-Kelley]{brandon2023striped}
William Brandon, Aniruddha Nrusimha, Kevin Qian, Zachary Ankner, Tian Jin,
  Zhiye Song, and Jonathan Ragan-Kelley.
\newblock Striped attention: Faster ring attention for causal transformers.
\newblock \emph{arXiv preprint arXiv:2311.09431}, 2023.
\newblock URL \url{https://arxiv.org/abs/2311.09431}.

\bibitem[Chen et~al.(2026)Chen, Liang, and Liu]{chen2026dflash}
Jian Chen, Yesheng Liang, and Zhijian Liu.
\newblock {DFlash: Block Diffusion for Flash Speculative Decoding}.
\newblock In \emph{Proceedings of the International Conference on Machine
  Learning ({ICML})}, 2026.
\newblock URL \url{https://arxiv.org/abs/2602.06036}.

\bibitem[Chen et~al.(2016)Chen, Xu, Zhang, and Guestrin]{chen2016checkpointing}
Tianqi Chen, Bing Xu, Chiyuan Zhang, and Carlos Guestrin.
\newblock Training deep nets with sublinear memory cost.
\newblock \emph{arXiv preprint arXiv:1604.06174}, 2016.
\newblock URL \url{https://arxiv.org/abs/1604.06174}.

\bibitem[Chowdhery et~al.(2022)Chowdhery, Narang, Devlin,
  et~al.]{chowdhery2022palm}
Aakanksha Chowdhery, Sharan Narang, Jacob Devlin, et~al.
\newblock {PaLM}: Scaling language modeling with pathways.
\newblock \emph{arXiv preprint arXiv:2204.02311}, 2022.
\newblock URL \url{https://arxiv.org/abs/2204.02311}.

\bibitem[Chowdhury et~al.(2024)Chowdhury, Aung, Shern, Jaffe, Sherburn,
  Starace, Mays, Dias, Aljubeh, Glaese, Jimenez, Yang, Ho, Patwardhan, Liu, and
  Madry]{chowdhury2024swebenchverified}
Neil Chowdhury, James Aung, Chan~Jun Shern, Oliver Jaffe, Dane Sherburn, Giulio
  Starace, Evan Mays, Rachel Dias, Marwan Aljubeh, Mia Glaese, Carlos~E.
  Jimenez, John Yang, Leyton Ho, Tejal Patwardhan, Kevin Liu, and Aleksander
  Madry.
\newblock Introducing {SWE}-bench verified, 2024.
\newblock URL \url{https://openai.com/index/introducing-swe-bench-verified/}.

\bibitem[Dao et~al.(2022)Dao, Fu, Ermon, Rudra, and
  R{\'e}]{dao2022flashattention}
Tri Dao, Daniel~Y. Fu, Stefano Ermon, Atri Rudra, and Christopher R{\'e}.
\newblock {FlashAttention}: Fast and memory-efficient exact attention with
  {IO}-awareness.
\newblock In \emph{Advances in Neural Information Processing Systems}, 2022.
\newblock URL \url{https://arxiv.org/abs/2205.14135}.

\bibitem[{DiffusionGemma Team}(2026)]{diffusiongemma2026report}
{DiffusionGemma Team}.
\newblock {DiffusionGemma} technical report.
\newblock \emph{arXiv preprint arXiv:2608.00146}, 2026.
\newblock URL \url{https://arxiv.org/abs/2608.00146}.

\bibitem[Dong et~al.(2025)Dong, Feng, Guessous, Liang, and
  He]{dong2025flexattention}
Juechu Dong, Boyuan Feng, Driss Guessous, Yanbo Liang, and Horace He.
\newblock {FlexAttention}: A programming model for generating fused attention
  variants.
\newblock In \emph{Proceedings of Machine Learning and Systems}, volume~7,
  2025.
\newblock URL
  \url{https://proceedings.mlsys.org/paper_files/paper/2025/hash/61a9278dfef5f871b5e472389f8d6fa1-Abstract-Conference.html}.

\bibitem[Dubey et~al.(2024)Dubey, Jauhri, Pandey, Kadian, Al-Dahle, Letman,
  Mathur, Schelten, Yang, Fan, et~al.]{dubey2024llama3}
Abhimanyu Dubey, Abhinav Jauhri, Abhinav Pandey, Abhishek Kadian, Ahmad
  Al-Dahle, Aiesha Letman, Akhil Mathur, Alan Schelten, Amy Yang, Angela Fan,
  et~al.
\newblock The {Llama 3} herd of models.
\newblock \emph{arXiv preprint arXiv:2407.21783}, 2024.
\newblock URL \url{https://arxiv.org/abs/2407.21783v1}.

\bibitem[Fu et~al.(2026)Fu, Whalen, Garg, et~al.]{fu2026nemotronlabsdiffusion}
Yonggan Fu, Lexington Whalen, Abhinav Garg, et~al.
\newblock {Nemotron-Labs-Diffusion}: A tri-mode language model unifying
  autoregressive, diffusion, and self-speculation decoding.
\newblock \emph{arXiv preprint arXiv:2607.05722}, 2026.
\newblock URL \url{https://arxiv.org/abs/2607.05722}.

\bibitem[Gu et~al.(2024)Gu, Sun, Hu, Huang, Chen, Xiong, Wang, Chen, Zhao,
  Fang, Wen, Zhang, Jin, and Liu]{gu2024loongtrain}
Diandian Gu, Peng Sun, Qinghao Hu, Ting Huang, Xun Chen, Yingtong Xiong,
  Guoteng Wang, Qiaoling Chen, Shangchun Zhao, Jiarui Fang, Yonggang Wen,
  Tianwei Zhang, Xin Jin, and Xuanzhe Liu.
\newblock {LoongTrain}: Efficient training of long-sequence {LLMs} with
  head-context parallelism.
\newblock \emph{arXiv preprint arXiv:2406.18485}, 2024.
\newblock URL \url{https://arxiv.org/abs/2406.18485}.

\bibitem[{Harbor Framework Team}(2026)]{Harbor_Framework}
{Harbor Framework Team}.
\newblock {Harbor}: A framework for evaluating and optimizing agents and models
  in container environments, August 2026.
\newblock URL \url{https://harborframework.com/}.

\bibitem[{Inco AI}(2026)]{inco2026dflash2}
{Inco AI}.
\newblock {DFlash 2: Keep Drafting Parallel}, August 2026.
\newblock URL \url{https://inco.ai/blog/dflash2/}.

\bibitem[Jacobs et~al.(2023)Jacobs, Tanaka, Zhang, Zhang, Song, Rajbhandari,
  and He]{jacobs2023ulysses}
Sam~Ade Jacobs, Masahiro Tanaka, Chengming Zhang, Minjia Zhang, Shuaiwen~Leon
  Song, Samyam Rajbhandari, and Yuxiong He.
\newblock {DeepSpeed-Ulysses}: System optimizations for enabling training of
  extreme long sequence transformer models.
\newblock \emph{arXiv preprint arXiv:2309.14509}, 2023.
\newblock URL \url{https://arxiv.org/abs/2309.14509}.

\bibitem[Jimenez et~al.(2024)Jimenez, Yang, Wettig, Yao, Pei, Press, and
  Narasimhan]{jimenez2024swebench}
Carlos~E. Jimenez, John Yang, Alexander Wettig, Shunyu Yao, Kexin Pei, Ofir
  Press, and Karthik~R. Narasimhan.
\newblock {SWE}-bench: Can language models resolve real-world {GitHub} issues?
\newblock In \emph{The Twelfth International Conference on Learning
  Representations}, 2024.
\newblock URL \url{https://openreview.net/forum?id=VTF8yNQM66}.

\bibitem[Kapfer et~al.(2025)Kapfer, Stine, Narasimhan, Mentzel, and
  Cand{\`e}s]{kapfer2025marlowe}
Craig Kapfer, Kurt Stine, Balasubramanian Narasimhan, Christopher Mentzel, and
  Emmanuel Cand{\`e}s.
\newblock {Marlowe}: Stanford's {GPU}-based computational instrument, 2025.
\newblock URL \url{https://doi.org/10.5281/zenodo.14751899}.

\bibitem[Korthikanti et~al.(2022)Korthikanti, Casper, Lym, McAfee, Andersch,
  Shoeybi, and Catanzaro]{korthikanti2022recomputation}
Vijay Korthikanti, Jared Casper, Sangkug Lym, Lawrence McAfee, Michael
  Andersch, Mohammad Shoeybi, and Bryan Catanzaro.
\newblock Reducing activation recomputation in large transformer models.
\newblock \emph{arXiv preprint arXiv:2205.05198}, 2022.
\newblock URL \url{https://arxiv.org/abs/2205.05198}.

\bibitem[Lepikhin et~al.(2020)Lepikhin, Lee, Xu, Chen, Firat, Huang, Krikun,
  Shazeer, and Chen]{lepikhin2020gshard}
Dmitry Lepikhin, HyoukJoong Lee, Yuanzhong Xu, Dehao Chen, Orhan Firat, Yanping
  Huang, Maxim Krikun, Noam Shazeer, and Zhifeng Chen.
\newblock {GShard}: Scaling giant models with conditional computation and
  automatic sharding.
\newblock \emph{arXiv preprint arXiv:2006.16668}, 2020.
\newblock URL \url{https://arxiv.org/abs/2006.16668}.

\bibitem[Leviathan et~al.(2023)Leviathan, Kalman, and
  Matias]{leviathan2023speculative}
Yaniv Leviathan, Matan Kalman, and Yossi Matias.
\newblock Fast inference from transformers via speculative decoding.
\newblock In \emph{Proceedings of the 40th International Conference on Machine
  Learning}, volume 202 of \emph{Proceedings of Machine Learning Research},
  pp.\  19274--19286, 2023.
\newblock URL \url{https://proceedings.mlr.press/v202/leviathan23a.html}.

\bibitem[Li et~al.(2025)Li, Zhu, Wang, Yin, Shi, Wang, Zhang, Huang, Zheng, and
  Zhang]{specforge2025}
Shenggui Li, Yikai Zhu, Chao Wang, Fan Yin, Shuai Shi, Yubo Wang, Yi~Zhang,
  Yingyi Huang, Haoshuai Zheng, and Yineng Zhang.
\newblock {SpecForge}: Train speculative decoding models effortlessly.
\newblock \url{https://github.com/sgl-project/specforge}, 2025.

\bibitem[Li et~al.(2026)Li, Wang, Zhu, Wang, Yin, Shi, Chen, Dong, Chen, Pan,
  et~al.]{li2026specforge}
Shenggui Li, Chao Wang, Yikai Zhu, Yubo Wang, Fan Yin, Shuai Shi, Yefei Chen,
  Xiaomin Dong, Qiaoling Chen, Jin Pan, et~al.
\newblock {SpecForge}: A flexible and efficient open-source training framework
  for speculative decoding.
\newblock \emph{arXiv preprint arXiv:2603.18567}, 2026.

\bibitem[Liang et~al.(2025)Liang, Liu, Wright, Constable, Gu, Huang, Zhang,
  Feng, Huang, Wang, Purandare, Nadathur, and Idreos]{pytorch_torchtitan}
Wanchao Liang, Tianyu Liu, Less Wright, Will Constable, Andrew Gu, Chien-Chin
  Huang, Iris Zhang, Wei Feng, Howard Huang, Junjie Wang, Sanket Purandare,
  Gokul Nadathur, and Stratos Idreos.
\newblock {TorchTitan}: One-stop {PyTorch} native solution for production ready
  {LLM} pre-training.
\newblock In \emph{The Thirteenth International Conference on Learning
  Representations}, 2025.
\newblock URL \url{https://openreview.net/forum?id=SFN6Wm7YBI}.

\bibitem[Liu et~al.(2023)Liu, Zaharia, and Abbeel]{liu2023ringattention}
Hao Liu, Matei Zaharia, and Pieter Abbeel.
\newblock Ring attention with blockwise transformers for near-infinite context.
\newblock \emph{arXiv preprint arXiv:2310.01889}, 2023.
\newblock URL \url{https://arxiv.org/abs/2310.01889}.

\bibitem[Milakov \& Gimelshein(2018)Milakov and Gimelshein]{milakov2018online}
Maxim Milakov and Natalia Gimelshein.
\newblock Online normalizer calculation for softmax.
\newblock \emph{arXiv preprint arXiv:1805.02867}, 2018.
\newblock URL \url{https://arxiv.org/abs/1805.02867}.

\bibitem[{NVIDIA}(n.d.)]{nvidia_megatron_cp}
{NVIDIA}.
\newblock {Context Parallel Package --- Megatron Core}.
\newblock Software documentation, n.d.
\newblock URL
  \url{https://docs.nvidia.com/megatron-core/developer-guide/latest/user-guide/features/context_parallel.html}.
\newblock Accessed September 5, 2026.

\bibitem[O'Donoghue \& Flennerhag(2026)O'Donoghue and
  Flennerhag]{odonoghue2026diffusiongemma}
Brendan O'Donoghue and Sebastian Flennerhag.
\newblock {DiffusionGemma}: 4x faster text generation.
\newblock Google Blog, June 2026.
\newblock URL
  \url{https://blog.google/innovation-and-ai/technology/developers-tools/diffusion-gemma-faster-text-generation/}.

\bibitem[{OpenThoughts-Agent team, Snorkel AI, Bespoke
  Labs}(2026)]{OpenThoughts-TBLite}
{OpenThoughts-Agent team, Snorkel AI, Bespoke Labs}.
\newblock {OpenThoughts-TBLite}: A high-signal benchmark for iterating on
  terminal agents.
\newblock \url{https://www.openthoughts.ai/blog/openthoughts-tblite}, February
  2026.

\bibitem[Paszke et~al.(2019)Paszke, Gross, Massa, Lerer, Bradbury, Chanan,
  Killeen, Lin, Gimelshein, Antiga, Desmaison, Kopf, Yang, DeVito, Raison,
  Tejani, Chilamkurthy, Steiner, Fang, Bai, and Chintala]{paszke2019pytorch}
Adam Paszke, Sam Gross, Francisco Massa, Adam Lerer, James Bradbury, Gregory
  Chanan, Trevor Killeen, Zeming Lin, Natalia Gimelshein, Luca Antiga, Alban
  Desmaison, Andreas Kopf, Edward Yang, Zachary DeVito, Martin Raison, Alykhan
  Tejani, Sasank Chilamkurthy, Benoit Steiner, Lu~Fang, Junjie Bai, and Soumith
  Chintala.
\newblock {PyTorch}: An imperative style, high-performance deep learning
  library.
\newblock In \emph{Advances in Neural Information Processing Systems},
  volume~32, 2019.
\newblock URL
  \url{https://papers.neurips.cc/paper_files/paper/2019/hash/bdbca288fee7f92f2bfa9f7012727740-Abstract.html}.

\bibitem[Peng et~al.(2026)Peng, Zhang, Lu, Cao, Lu, Lin, Han, and
  Sun]{peng2026litecoderterminal}
Xiaoxuan Peng, Kaiqi Zhang, Xinyu Lu, Boxi Cao, Yaojie Lu, Hongyu Lin, Xianpei
  Han, and Le~Sun.
\newblock {LiteCoder-Terminal}: Scaling long-horizon terminal environments for
  learning language agents.
\newblock \emph{arXiv preprint arXiv:2605.29559}, 2026.
\newblock URL \url{https://arxiv.org/abs/2605.29559}.

\bibitem[{Qwen Team}(2026{\natexlab{a}})]{qwen2026qwen35}
{Qwen Team}.
\newblock {Qwen3.5}: Towards native multimodal agents, February
  2026{\natexlab{a}}.
\newblock URL \url{https://qwen.ai/blog?id=qwen3.5}.

\bibitem[{Qwen Team}(2026{\natexlab{b}})]{qwen2026qwen38}
{Qwen Team}.
\newblock {Qwen3.8-Max}: A new bar for coding and cowork, August
  2026{\natexlab{b}}.
\newblock URL \url{https://qwen.ai/blog?id=qwen3.8}.

\bibitem[Rajbhandari et~al.(2019)Rajbhandari, Rasley, Ruwase, and
  He]{rajbhandari2019zero}
Samyam Rajbhandari, Jeff Rasley, Olatunji Ruwase, and Yuxiong He.
\newblock {ZeRO}: Memory optimizations toward training trillion parameter
  models.
\newblock \emph{arXiv preprint arXiv:1910.02054}, 2019.
\newblock URL \url{https://arxiv.org/abs/1910.02054}.

\bibitem[Shoeybi et~al.(2019)Shoeybi, Patwary, Puri, LeGresley, Casper, and
  Catanzaro]{shoeybi2019megatron}
Mohammad Shoeybi, Mostofa Patwary, Raul Puri, Patrick LeGresley, Jared Casper,
  and Bryan Catanzaro.
\newblock {Megatron-LM}: Training multi-billion parameter language models using
  model parallelism.
\newblock \emph{arXiv preprint arXiv:1909.08053}, 2019.
\newblock URL \url{https://arxiv.org/abs/1909.08053}.

\bibitem[Vaswani et~al.(2017)Vaswani, Shazeer, Parmar, Uszkoreit, Jones, Gomez,
  Kaiser, and Polosukhin]{vaswani2017attention}
Ashish Vaswani, Noam Shazeer, Niki Parmar, Jakob Uszkoreit, Llion Jones,
  Aidan~N. Gomez, Lukasz Kaiser, and Illia Polosukhin.
\newblock Attention is all you need.
\newblock \emph{arXiv preprint arXiv:1706.03762}, 2017.
\newblock URL \url{https://arxiv.org/abs/1706.03762}.

\bibitem[Wang et~al.(2025)Wang, Li, Song, Xu, Tang, Zhuge, Pan, Song, Li,
  Singh, Tran, Li, Ma, Zheng, Qian, Shao, Muennighoff, Zhang, Hui, Lin,
  Brennan, Peng, Ji, and Neubig]{wang2025openhands}
Xingyao Wang, Boxuan Li, Yufan Song, Frank~F. Xu, Xiangru Tang, Mingchen Zhuge,
  Jiayi Pan, Yueqi Song, Bowen Li, Jaskirat Singh, Hoang~H. Tran, Fuqiang Li,
  Ren Ma, Mingzhang Zheng, Bill Qian, Yanjun Shao, Niklas Muennighoff, Yizhe
  Zhang, Binyuan Hui, Junyang Lin, Robert Brennan, Hao Peng, Heng Ji, and
  Graham Neubig.
\newblock {OpenHands}: An open platform for {AI} software developers as
  generalist agents.
\newblock In \emph{The Thirteenth International Conference on Learning
  Representations}, 2025.
\newblock URL \url{https://openreview.net/forum?id=OJd3ayDDoF}.

\bibitem[Wu et~al.(2025)Wu, Zhang, Xue, Diao, Fu, Liu, Molchanov, Luo, Han, and
  Xie]{wu2025fastdllmv2}
Chengyue Wu, Hao Zhang, Shuchen Xue, Shizhe Diao, Yonggan Fu, Zhijian Liu,
  Pavlo Molchanov, Ping Luo, Song Han, and Enze Xie.
\newblock {Fast-dLLM v2}: Efficient block-diffusion {LLM}.
\newblock \emph{arXiv preprint arXiv:2509.26328}, 2025.
\newblock URL \url{https://arxiv.org/abs/2509.26328}.

\bibitem[Yang et~al.(2024)Yang, Jimenez, Wettig, Lieret, Yao, Narasimhan, and
  Press]{yang2024sweagent}
John Yang, Carlos~E. Jimenez, Alexander Wettig, Kilian Lieret, Shunyu Yao,
  Karthik Narasimhan, and Ofir Press.
\newblock {SWE-agent}: Agent-computer interfaces enable automated software
  engineering.
\newblock \emph{arXiv preprint arXiv:2405.15793}, 2024.
\newblock URL \url{https://arxiv.org/abs/2405.15793}.

\bibitem[Yao et~al.(2023)Yao, Zhao, Yu, Du, Shafran, Narasimhan, and
  Cao]{yao2023react}
Shunyu Yao, Jeffrey Zhao, Dian Yu, Nan Du, Izhak Shafran, Karthik Narasimhan,
  and Yuan Cao.
\newblock {ReAct}: Synergizing reasoning and acting in language models.
\newblock In \emph{International Conference on Learning Representations}, 2023.
\newblock URL \url{https://arxiv.org/abs/2210.03629}.

\bibitem[Zadouri et~al.(2026)Zadouri, Hoehnerbach, Shah, Liu, Thakkar, and
  Dao]{zadouri2026flashattention4}
Ted Zadouri, Markus Hoehnerbach, Jay Shah, Timmy Liu, Vijay Thakkar, and Tri
  Dao.
\newblock {FlashAttention-4}: Algorithm and kernel pipelining co-design for
  asymmetric hardware scaling.
\newblock \emph{arXiv preprint arXiv:2603.05451}, 2026.
\newblock URL \url{https://arxiv.org/abs/2603.05451}.

\end{thebibliography}

\clearpage
\appendix
\section{Correctness of Context-Sharded Block Parallelism}
\label{app:correctness}

\begin{proposition}[Exactness of CSBP]
\label{prop:csbp-exactness}
Let $\{\mathcal{B}_r\}_{r=1}^{P}$ partition the $B$ target blocks, and let the
clean token rows be partitioned arbitrarily across the same $P$ ranks. Given
any fixed input sequence, sampled diffusion times and corruptions, per-token
model randomness, and loss normalization, CSBP computes, for every
parameter value $\theta$, the same global sampled loss and parameter gradient
as the reference length-$2L$ BDLM computation, up to floating-point reduction
order.
\end{proposition}

\begin{proof}
Fix the input sequence, sampled diffusion times and corruptions, model
randomness, and loss normalization. Let
$\mathcal L_b^{\mathrm{ref}}(\theta)$ and
$\mathcal L_b^{\mathrm{CSBP}}(\theta)$ be the resulting sampled losses for
block $b$.
For rank $r$, define
\begin{equation}
    \mathcal L_r^{\mathrm{CSBP}}
    =\sum_{b\in\mathcal B_r}\mathcal L_b^{\mathrm{CSBP}},
    \qquad
    \mathcal L^{\mathrm{CSBP}}
    =\sum_{r=1}^{P}\mathcal L_r^{\mathrm{CSBP}},
    \qquad
    \mathcal L^{\mathrm{ref}}
    =\sum_{b=1}^{B}\mathcal L_b^{\mathrm{ref}}.
    \label{eq:correctness-loss-definitions}
\end{equation}

Let $H_d^{\mathrm{clean}}$ and $H_{b,d}^{\mathrm{corr}}$ denote the clean
representations and the corrupted representations of block $b$ after $d$ of
the model's $D$ transformer layers. Thus $H_0$ is the embedding output and
$H_D$ enters the output projection. We use $H_d$ to denote all clean and
corrupted representations at depth $d$, and $\theta_d$ to denote the
parameters of layer $d$ for $1\leq d\leq D$.

\begin{enumerate}
\item
\textbf{Reference dependencies.}
The length-$2L$ construction and block-causal mask described in
Section~\ref{sec:background} determine the visible K/V rows. At layer $d$, a
clean query at position $q$ attends to clean positions $1,\ldots,q$. The
corrupted queries in block $b$ attend to their own corrupted block and the
clean positions in preceding blocks. We write $K_{\leq q,d}^{\mathrm{clean}}$
and $V_{\leq q,d}^{\mathrm{clean}}$ for the clean K/V rows at positions
$1,\ldots,q$. The reference attention is
\begin{align}
    O_{q,d}^{\mathrm{clean,ref}}
    &= \operatorname{Attn}\!\left(
        Q_{q,d}^{\mathrm{clean}},
        K_{\leq q,d}^{\mathrm{clean}},
        V_{\leq q,d}^{\mathrm{clean}}
      \right),
    \label{eq:correctness-clean-attention}
    \\
    O_{b,d}^{\mathrm{corr,ref}}
    &= \operatorname{Attn}\!\left(
        Q_{b,d}^{\mathrm{corr}},
        [K_{b,d}^{\mathrm{corr}};K_{<b,d}^{\mathrm{clean}}],
        [V_{b,d}^{\mathrm{corr}};V_{<b,d}^{\mathrm{clean}}]
      \right).
    \label{eq:correctness-corrupted-attention}
\end{align}
Equation~\ref{eq:correctness-corrupted-attention} implies, for every
$Z\in\{Q,K,V\}$,
\begin{equation}
    \frac{\partial O_{b,d}^{\mathrm{corr,ref}}}
         {\partial Z_{b',d}^{\mathrm{corr}}}
    =0
    \qquad (b'\neq b).
    \label{eq:correctness-block-dependency}
\end{equation}

\item
\textbf{Exact sharded attention.}
Let $C_r$ be the set of clean token positions stored on rank $r$. The clean
shards satisfy
\begin{equation}
    \bigcup_{r=1}^{P} C_r=\{1,\ldots,L\},
    \qquad
    C_r\cap C_{r'}=\varnothing \quad (r\neq r').
    \label{eq:correctness-clean-partition}
\end{equation}
Consider one query vector $\mathbf q$. Let $S$ be its visible key positions,
partitioned into disjoint nonempty sets $S_1,\ldots,S_J$. For the key and value
rows $\mathbf k_i$ and $\mathbf v_i$, let
$a_i=\mathbf q\mathbf k_i^\top/\sqrt{d_h}$ and
$a_{\max}=\max_{i\in S}a_i$. Since the sets $S_j$ partition $S$,
\begin{equation}
    \operatorname{Attn}(\mathbf q,K_S,V_S)
    =
    \frac{\displaystyle\sum_{j=1}^{J}\sum_{i\in S_j}
          e^{a_i-a_{\max}}\mathbf v_i}
         {\displaystyle\sum_{j=1}^{J}\sum_{i\in S_j}e^{a_i-a_{\max}}}.
    \label{eq:correctness-online-softmax-merge}
\end{equation}
Online softmax evaluates these partitioned sums exactly
\citep{milakov2018online,dao2022flashattention}. For a clean query at position
$q$, the nonempty sets $C_r\cap\{1,\ldots,q\}$ partition its visible K/V rows.
For a corrupted query in block $b$, its local corrupted block together with
the nonempty clean sets $C_r\cap\{1,\ldots,(b-1)M\}$ partition its visible
K/V rows. These clean sets are empty for $b=1$.
CSBP therefore supplies exactly the rows in
Equations~\ref{eq:correctness-clean-attention} and
\ref{eq:correctness-corrupted-attention}.

\item
\textbf{Forward induction.}
The common inputs and embedding computation give
$H_0^{\mathrm{CSBP}}=H_0^{\mathrm{ref}}$. Assume
$H_{d-1}^{\mathrm{CSBP}}=H_{d-1}^{\mathrm{ref}}$ for some
$1\leq d\leq D$. Their Q/K/V projections agree, and
Equation~\ref{eq:correctness-online-softmax-merge} gives
\begin{equation}
    \begin{aligned}
      O_{q,d}^{\mathrm{clean,CSBP}}&=O_{q,d}^{\mathrm{clean,ref}}
      && \forall q,\\
      O_{b,d}^{\mathrm{corr,CSBP}}=O_{b,d}^{\mathrm{corr,ref}}
      && \forall b.
    \end{aligned}
    \label{eq:correctness-forward-step}
\end{equation}
The remaining operations in layer $d$ are unchanged, so
$H_d^{\mathrm{CSBP}}=H_d^{\mathrm{ref}}$. Induction over
$d=1,\ldots,D$ gives $H_D^{\mathrm{CSBP}}=H_D^{\mathrm{ref}}$.

\item
\textbf{Loss.}
The output projection and block losses are unchanged. The forward induction
therefore gives
$\mathcal L_b^{\mathrm{CSBP}}=\mathcal L_b^{\mathrm{ref}}$ for every block.
Since $\{\mathcal B_r\}_{r=1}^{P}$ partitions the target blocks,
\begin{equation}
    \mathcal L^{\mathrm{CSBP}}
    =\sum_{r=1}^{P}\sum_{b\in\mathcal B_r}\mathcal L_b^{\mathrm{CSBP}}
    =\sum_{b=1}^{B}\mathcal L_b^{\mathrm{ref}}
    =\mathcal L^{\mathrm{ref}}.
    \label{eq:correctness-sampled-loss}
\end{equation}

\item
\textbf{Backward pass.}
The equal final representations and common loss functions give the same
derivatives at the output of the model. Let $\mathcal L$ denote the global
sampled loss of either execution. For
$Z\in\{Q,K,V\}$,
Equation~\ref{eq:correctness-block-dependency} and the chain rule give
\begin{equation}
    \frac{\partial \mathcal L}{\partial Z_{b,d}^{\mathrm{corr}}}
    = \sum_{b'=1}^{B}
      \left(
      \frac{\partial O_{b',d}^{\mathrm{corr}}}
           {\partial Z_{b,d}^{\mathrm{corr}}}
      \right)^{\!\top}
      \frac{\partial \mathcal L}{\partial O_{b',d}^{\mathrm{corr}}}
    = \left(
      \frac{\partial O_{b,d}^{\mathrm{corr}}}
           {\partial Z_{b,d}^{\mathrm{corr}}}
      \right)^{\!\top}
      \frac{\partial \mathcal L}{\partial O_{b,d}^{\mathrm{corr}}}.
    \label{eq:correctness-corrupted-gradient}
\end{equation}
Thus corrupted Q/K/V derivatives depend only on the queries from the same
block and are computed on its owner. Let $Z_{r,d}^{\mathrm{clean}}$ denote the
clean K/V rows at positions $C_r$, and let $O_{q,d}$ denote the attention
output for query row $q$. The chain rule gives
\begin{equation}
    \frac{\partial \mathcal L}{\partial Z_{r,d}^{\mathrm{clean}}}
    = \sum_q
      \left(
        \frac{\partial O_{q,d}}{\partial Z_{r,d}^{\mathrm{clean}}}
      \right)^{\!\top}
      \frac{\partial \mathcal L}{\partial O_{q,d}},
    \qquad Z\in\{K,V\},
    \label{eq:correctness-clean-gradient}
\end{equation}
where the sum is over all clean and corrupted query rows. A term is zero when
the query cannot attend to a position in $C_r$. Exact CP backward accumulates
this sum on rank $r$. For every clean query row $q$,
\begin{equation}
    \frac{\partial \mathcal L}{\partial Q_{q,d}^{\mathrm{clean}}}
    = \left(
        \frac{\partial O_{q,d}^{\mathrm{clean}}}
             {\partial Q_{q,d}^{\mathrm{clean}}}
      \right)^{\!\top}
      \frac{\partial \mathcal L}{\partial O_{q,d}^{\mathrm{clean}}},
    \label{eq:correctness-query-gradient}
\end{equation}
which is computed on the query owner. Because the two executions have the same
forward values and output derivatives, Equations~\ref{eq:correctness-corrupted-gradient}--\ref{eq:correctness-query-gradient}
give the same Q/K/V derivatives. Backpropagation through the unchanged
remainder of layer $d$ therefore satisfies
\begin{equation}
    \begin{aligned}
      \frac{\partial \mathcal L^{\mathrm{CSBP}}}
           {\partial H_{d-1}^{\mathrm{CSBP}}}
      &=
      \left(
        \frac{\partial H_d^{\mathrm{CSBP}}}
             {\partial H_{d-1}^{\mathrm{CSBP}}}
      \right)^{\!\top}
      \frac{\partial \mathcal L^{\mathrm{CSBP}}}
           {\partial H_d^{\mathrm{CSBP}}} \\
      &=
      \left(
        \frac{\partial H_d^{\mathrm{ref}}}
             {\partial H_{d-1}^{\mathrm{ref}}}
      \right)^{\!\top}
      \frac{\partial \mathcal L^{\mathrm{ref}}}
           {\partial H_d^{\mathrm{ref}}}
      =
      \frac{\partial \mathcal L^{\mathrm{ref}}}
           {\partial H_{d-1}^{\mathrm{ref}}}, \\[4pt]
      \nabla_{\theta_d}\mathcal L^{\mathrm{CSBP}}
      &=
      \left(
        \frac{\partial H_d^{\mathrm{CSBP}}}{\partial \theta_d}
      \right)^{\!\top}
      \frac{\partial \mathcal L^{\mathrm{CSBP}}}
           {\partial H_d^{\mathrm{CSBP}}} \\
      &=
      \left(
        \frac{\partial H_d^{\mathrm{ref}}}{\partial \theta_d}
      \right)^{\!\top}
      \frac{\partial \mathcal L^{\mathrm{ref}}}
           {\partial H_d^{\mathrm{ref}}}
      =\nabla_{\theta_d}\mathcal L^{\mathrm{ref}}.
    \end{aligned}
    \label{eq:correctness-backward-step}
\end{equation}
Because the output projection is identical, its parameter gradient and
derivative with respect to $H_D$ agree, providing the base case. Reverse
induction through $d=D,\ldots,1$ gives equal transformer-layer gradients and
equal derivatives with respect to $H_0$. The common embedding computation then
gives equal embedding gradients. Finally, linearity of differentiation and
Equation~\ref{eq:correctness-loss-definitions} give
\begin{equation}
    \sum_{r=1}^{P}
    \nabla_\theta \mathcal L_r^{\mathrm{CSBP}}
    =\nabla_\theta \mathcal L^{\mathrm{CSBP}}
    =\nabla_\theta \mathcal L^{\mathrm{ref}}.
    \label{eq:correctness-sampled-gradient}
\end{equation}
\end{enumerate}
Equations~\ref{eq:correctness-sampled-loss} and
\ref{eq:correctness-sampled-gradient} hold for every sampled realization.
Taking expectations also gives the same BDLM objective and objective gradient.
\end{proof}

\paragraph{Communication consequence.}
Under CSBP, only clean K/V and their gradients cross ranks during attention,
while corrupted Q/K/V and their gradients remain on the rank assigned their
block.

\section{Load Balancing}
\label{app:load-balancing}

Both attention workloads in CSBP increase with block position. Later
target blocks attend to longer clean prefixes, and later clean query blocks
attend to longer causal prefixes. Because both costs grow with the same block
index, the same early--late pairing balances them simultaneously.

Assume $B$ equal-sized blocks of $M=L/B$ tokens. Let
$W_b^{\mathrm{target}}$ and $W_b^{\mathrm{clean}}$ denote the valid query--key
pairs for target-block and clean causal attention at block $b$. The $M$ target
queries in block $b$ attend to $(b-1)M$ clean-prefix keys and $M$ keys from
their own corrupted block, for $bM$ visible keys. The $M$ clean queries attend
to the $b-1$ preceding clean blocks and causally within their own block. Their
work is therefore
\begin{align}
    W_b^{\mathrm{target}}
    &= bM^2,
    &
    W_b^{\mathrm{clean}}
    &= (b-1)M^2+\frac{M(M+1)}{2}.
    \label{eq:block-work-by-position}
\end{align}

Dual-end BP pairs target block $b$ with block $B+1-b$. Clean-sequence zigzag
sharding uses the same early--late pairing~\citep{dubey2024llama3}. The work of
each pair is
\begin{align}
    W_b^{\mathrm{target}}+W_{B+1-b}^{\mathrm{target}}
    &= (B+1)M^2,
    &
    W_b^{\mathrm{clean}}+W_{B+1-b}^{\mathrm{clean}}
    &= BM^2+M.
    \label{eq:aligned-pair-work}
\end{align}
Both sums are independent of $b$, so every early--late pair carries the same
amount of target-block work and the same amount of clean causal work. When $B$
is divisible by $2P$, assigning $B/(2P)$ pairs to each rank gives
\begin{align}
    W_r^{\mathrm{target}}
    &= \frac{B(B+1)M^2}{2P},
    &
    W_r^{\mathrm{clean}}
    &= \frac{L(L+1)}{2P},
    \qquad 1\leq r\leq P.
    \label{eq:aligned-rank-work}
\end{align}
Using the same pairs for target-block placement and clean zigzag sharding gives
every rank equal target-block and clean causal attention work. Attention
backward uses the same valid query--key pairs, so these pair counts remain
balanced. When exact divisibility does not hold, we distribute the complete
pairs as evenly as possible and assign any unpaired middle block to the
currently least-loaded rank.

\clearpage
\section{CSBP Algorithm and Execution Schedule}
\label{app:execution-schedule}

Algorithm~\ref{alg:csbp-attention} summarizes one CSBP layer. Each rank stores
one clean-sequence shard and the complete corrupted blocks in $\mathcal B_r$.
All ranks execute concurrently: CP supplies clean K/V, while each block's
corrupted Q/K/V remain local.

\begin{algorithm}[H]
\small
\caption{CSBP forward pass on rank $r$}
\label{alg:csbp-attention}
\begin{algorithmic}[1]
\State \textbf{Input:} Clean shard $(Q_r,K_r,V_r)$ and
  $\{(Q_b,K_b,V_b):b\in\mathcal B_r\}$
\State \textbf{Output:} $O_r$ and $\{O_b:b\in\mathcal B_r\}$
\State All-gather the clean K/V shards to form
  $(K^{\mathrm{clean}},V^{\mathrm{clean}})$ on each rank
\Statex
\State \textbf{Clean-sequence attention (CP)}
\State Compute clean-query attention under the block-causal mask
\State $O_r\gets\operatorname{softmax}\!\left(
  Q_r(K^{\mathrm{clean}})^\top/\sqrt{d_h}+\mathcal M_r\right)
  V^{\mathrm{clean}}$
\Statex
\State \textbf{Assigned corrupted blocks (BP)}
\ForAll{$b\in\mathcal B_r$}
  \State Compute within-block attention from the local $Q_b,K_b,V_b$
  \State Compute clean-prefix attention from the gathered
    $K_{<b}^{\mathrm{clean}},V_{<b}^{\mathrm{clean}}$
  \State Merge the two outputs and LSEs by online softmax to obtain
  \State $O_b\gets\operatorname{Attn}\!\left(
    Q_b,[K_b;K_{<b}^{\mathrm{clean}}],
    [V_b;V_{<b}^{\mathrm{clean}}]\right)$
    \Comment{Equation~\ref{eq:csbp-block-attention}}
\EndFor
\State \Return $O_r$ and $\{O_b:b\in\mathcal B_r\}$
\end{algorithmic}
\end{algorithm}

Algorithm~\ref{alg:csbp-backward} gives the exact backward pass. Rank $r$
keeps the corrupted Q/K/V gradients for its assigned blocks, while CP sums the
clean K/V gradients contributed by all clean and corrupted queries.

\begin{algorithm}[H]
\small
\caption{CSBP attention backward on rank $r$}
\label{alg:csbp-backward}
\begin{algorithmic}[1]
\State \textbf{Input:} Saved forward tensors and
  $\{\mathcal L_b:b\in\mathcal B_r\}$
\State \textbf{Output:} $(dQ_r,dK_r,dV_r)$ and
  $\{(dQ_b,dK_b,dV_b):b\in\mathcal B_r\}$
\Statex
\State \textbf{Clean-query backward (CP)}
\State Backpropagate $O_r$; retain $dQ_r$ and the clean K/V contributions
\Statex
\State \textbf{Assigned corrupted blocks (local backward)}
\ForAll{$b\in\mathcal B_r$}
  \State Backpropagate $\mathcal L_b$ through $O_b$
  \State $(dQ_b,dK_b,dV_b)
    \gets\nabla_{(Q_b,K_b,V_b)}\mathcal L_b$
    \Comment{kept on rank $r$}
  \State Add this block's clean-prefix K/V gradients to the CP reduction
\EndFor
\Statex
\State \textbf{Clean-gradient reduction (CP backward)}
\State Reduce-scatter the clean K/V contributions across ranks
\For{$Z\in\{K,V\}$}
  \State $\displaystyle dZ_r\gets
    \sum_q\left(\frac{\partial O_q}{\partial Z_r}\right)^{\!\top}
    \frac{\partial\mathcal L}{\partial O_q}$
    \Comment{Equation~\ref{eq:correctness-clean-gradient}}
\EndFor
\State \Return $(dQ_r,dK_r,dV_r)$ and
  $\{(dQ_b,dK_b,dV_b):b\in\mathcal B_r\}$
\end{algorithmic}
\end{algorithm}

\clearpage
\paragraph{Efficient fused operator.}
Implementation~\ref{impl:csbp-forward} shows the production forward path. It
starts an asynchronous all-gather of clean K/V, then computes all local
corrupted blocks with one block-diagonal attention launch while communication
runs. After the all-gather, one masked fused-attention launch computes the
clean-K/V contribution for every query. The mask uses each gathered key's
original clean token position, so the kernel consumes the gathered K/V directly
without a reordering copy.

Each attention launch returns a normalized output $O_j$ and row-wise
log-sum-exp (LSE) $z_j$.  For two disjoint key sets, the CUDA kernel sets
$m=\max(z_1,z_2)$, $w_j=e^{z_j-m}$, and computes
\begin{equation}
  O=\frac{w_1O_1+w_2O_2}{w_1+w_2},\qquad
  z=m+\log(w_1+w_2).
  \label{eq:csbp-lse-merge}
\end{equation}
This is Equation~\ref{eq:correctness-online-softmax-merge} applied to the two
attention launches.  The implementation stores the numerator, $m$, and the
normalizer $w_1+w_2$ in FP32 until the final division.

\begin{lstlisting}[
  style=implementation,
  language=Python,
  caption={CSBP forward schedule and CUDA LSE merge.},
  label={impl:csbp-forward}
]
@cuda_kernel
def merge_lse_kernel(numerator, m, normalizer,
                     O_block, z_block):
    # One thread handles one (corrupted-query row, head-dimension) element.
    q, d = cuda_grid_2d()
    if q >= len(z_block):
        return
    m_new = max(m[q], z_block[q])
    w_acc = exp(m[q] - m_new) if isfinite(m[q]) else 0.0
    w_block = exp(z_block[q] - m_new) if isfinite(z_block[q]) else 0.0
    numerator[q, d] = (w_acc * numerator[q, d]
                       + w_block * O_block[q, d])
    if d == 0:
        normalizer[q] = w_acc * normalizer[q] + w_block
        m[q] = m_new


def csbp_collective_forward(Q, K_block, V_block,
                            K_clean_r, V_clean_r, masks, scale):
    # 1. Start the clean K/V all-gather.
    work, gathered = async_all_gather(K_clean_r, V_clean_r)

    # 2. One block-diagonal launch covers every block owned by this rank.
    O_block, LSE_block = block_diagonal_attention(
        Q[:masks.corrupted_rows], K_block, V_block, scale)

    # 3. Use the gathered clean K/V directly in one masked launch.
    work.wait()
    K_clean, V_clean = view_all_gathered_clean_kv(gathered)
    clean_mask = block_causal_mask(masks)
    O_clean, LSE_clean = masked_clean_attention(
        Q, K_clean, V_clean, clean_mask, scale)

    # 4. Custom CUDA kernels apply the LSE rule above in FP32.
    numerator = O_clean.float()
    m = LSE_clean.float()
    normalizer = isfinite(m).float()
    launch(merge_lse_kernel, numerator, m, normalizer,
           O_block, LSE_block)
    O, LSE = empty_outputs(Q)
    launch(normalize_output_kernel,
           numerator, m, normalizer, O, LSE)
    save_for_backward(Q, K_block, V_block, K_clean, V_clean, O, LSE)
    return O
\end{lstlisting}

\clearpage
Implementation~\ref{impl:csbp-backward} shows the fused packed backward. One
masked kernel reuses the saved output and LSE to compute gradients for the local
corrupted K/V and gathered clean K/V together. The corrupted gradients remain
local, while only the clean K/V gradients are reduce-scattered to their CP
ranks.

\begin{lstlisting}[
  style=implementation,
  language=Python,
  caption={Fused CSBP backward. One masked kernel computes the local corrupted-block and clean-context gradients.},
  label={impl:csbp-backward}
]
def csbp_fused_backward(dO, saved, masks, scale):
    Q, K_block, V_block, K_clean, V_clean, O, LSE = saved

    # 1. Pack local corrupted K/V before the gathered clean K/V.
    K_packed, V_packed = pack_local_clean_kv(
        K_block, V_block, K_clean, V_clean)

    # 2. One masked backward reuses the exact forward output and LSE.
    dQ, dK_packed, dV_packed = masked_attention_backward(
        Q, K_packed, V_packed, O, LSE, dO,
        masks.packed, scale)

    # 3. Split local gradients from the clean gradients that communicate.
    dK_block, dK_clean = split_local_clean_gradients(
        dK_packed, masks.corrupted_rows)
    dV_block, dV_clean = split_local_clean_gradients(
        dV_packed, masks.corrupted_rows)

    # 4. Return each rank's clean K/V-gradient shard.
    dK_clean_r, dV_clean_r = reduce_scatter_clean_gradients(
        dK_clean, dV_clean)
    return dQ, dK_block, dV_block, dK_clean_r, dV_clean_r
\end{lstlisting}

\section{AR-to-BDLM Conversion Methodology}
\label{app:fast-dllm-v2-objective}

Fast-dLLM v2 adapts a pretrained causal LM to block diffusion while retaining
next-token prediction~\citep{wu2025fastdllmv2}. It samples a blockwise binary
mask $m^{(0)}$ and constructs a second noisy view with the complementary mask
$m^{(1)}=1-m^{(0)}$ on every supervised token. If $h_i^{(v)}$ is the hidden
representation at position $i$ in view $v\in\{0,1\}$, let $m_i^{(v)}=1$ when
$x_i$ is masked and let $\mathcal{S}$ contain the supervised positions other
than the first token. The sampled objective is
\begin{equation}
    \widehat{\mathcal{L}}_{\mathrm{AR\mbox{-}to\mbox{-}BDLM}}
    =
    -\sum_{v=0}^{1}\sum_{i\in\mathcal{S}}
    m_i^{(v)}\log p_\theta\!\left(x_i\mid h_{i-1}^{(v)}\right),
    \qquad
    m_i^{(0)}+m_i^{(1)}=1.
    \label{eq:fast-dllm-v2-objective}
\end{equation}
The complementary views supervise every eligible next token exactly once.
The one-token shift uses the preceding query row to predict $x_i$, matching a
causal LM head.

The objective preserves the two properties used by block parallelism. Each
view uses the length-$2L$ clean-plus-corrupted computation, and queries in
block $b$ use corrupted K/V only from block $b$ of that view and clean K/V
from preceding clean blocks,
\begin{equation}
    O_{b,v}
    =\operatorname{Attn}\!\left(
      Q_{b,v},
      [K_{b,v};K_{<b,v}^{\mathrm{clean}}],
      [V_{b,v};V_{<b,v}^{\mathrm{clean}}]
    \right).
    \label{eq:fast-dllm-v2-attention}
\end{equation}
The two views are separate batch elements and introduce no cross-view
attention. The label shift also leaves these attention dependencies unchanged.
We assign each shifted loss term to the block containing its prediction row,
and group the terms from both views into the corresponding block loss. The
objective therefore retains the decomposition
$\widehat{\mathcal{L}}_{\mathrm{AR\mbox{-}to\mbox{-}BDLM}}
=\sum_{b=1}^{B}\mathcal{L}_b$.

CSBP therefore applies directly. BP assigns each complete corrupted
computation and loss for block $b$ in both views to one rank, while CP shards
the clean sequence of both views across the same ranks.
Corrupted Q/K/V and their gradients remain on the block owner, while clean K/V
gradients are accumulated on their CP shard owners.
Proposition~\ref{prop:csbp-exactness}
applies to each view, and summing the complementary-view losses preserves its
loss and gradient equalities.

\clearpage
\section{DFlash2 Drafter Training Methodology}
\label{app:dflash2-objective}

DFlash trains a lightweight block-diffusion drafter for speculative
decoding~\citep{chen2026dflash}. Given a clean sequence $x$, hidden features
$H^{\mathrm{tar}}$ are extracted from a frozen autoregressive target model.
For each sampled block $b$, the token at position $a_b$ is visible and the next
$M-1$ tokens are masked and predicted in one pass. Queries attend
bidirectionally within the block and to target-feature K/V preceding $a_b$;
different draft blocks do not attend to one another.

Let
$\widetilde{x}^{(b)}=(x_{a_b},\mathtt{[MASK]},\ldots,\mathtt{[MASK]})$
denote the input for block $b$. One draft-attention layer computes
\begin{equation}
    O_b
    =
    \operatorname{Attn}\!\left(
      Q_b,
      [K_b;K_{<a_b}^{\mathrm{tar}}],
      [V_b;V_{<a_b}^{\mathrm{tar}}]
    \right).
    \label{eq:dflash2-block-attention}
\end{equation}
Following DFlash2~\citep{inco2026dflash2}, our drafter uses the architecture in
the open-source SpecForge implementation~\citep{specforge2025}. It retains this
one-pass backbone and uses block-local two-tap dynamic convolutions and a
candidate-path selector. Its sampled objective is the sum of a draft-token loss
and a selector loss for each block:
\begin{equation}
    \widehat{\mathcal{L}}_{\mathrm{DFlash2}}
    = \frac{1}{Z}\sum_b \mathcal{L}_b,
    \qquad
    \mathcal{L}_b
    = \mathcal{L}^{\mathrm{draft}}_b
      + \alpha_{\mathrm{sel}}\mathcal{L}^{\mathrm{sel}}_b,
    \label{eq:dflash2-training-objective}
\end{equation}
where $Z$ is the sum of valid token weights and $\alpha_{\mathrm{sel}}$ weights
the selector term.

This blockwise objective makes CSBP directly applicable. Let
$\mathcal B_r$ be the blocks assigned to rank $r$; these sets are disjoint and
together contain every sampled block. Since $\mathcal L_b$ depends only on block
$b$ and its target-feature prefix, each rank evaluates its assigned terms, and
\begin{equation}
    \sum_{r=1}^{P}
    \nabla_\theta\!\left(
      \frac{1}{Z}\sum_{b\in\mathcal B_r}\mathcal L_b
    \right)
    = \nabla_\theta\widehat{\mathcal L}_{\mathrm{DFlash2}}.
    \label{eq:dflash2-gradient-decomposition}
\end{equation}
Thus, assigning complete blocks to ranks preserves the loss and gradient up to
floating-point reduction order. It also gives DFlash2 an additional source of
savings beyond full-model SFT and AR-to-BDLM conversion. With $P$ context
ranks, conventional CP repeats each block-local drafter path $P$ times, while
CSBP evaluates it once. Full-model training has no separate drafter path and
therefore no corresponding $P$-to-one reduction. The repeated DFlash2 work
includes the decoder, vocabulary objective, output-head recomputation, and
selector; target-feature processing remains context sharded. At 512K, removing
two-way replication reduces measured GEMM work by $1.92\times$, and higher
execution efficiency raises the throughput gain to $2.48\times$. At 1M,
removing eight-way replication permits four data-parallel replicas, producing
the $7.59\times$ aggregate throughput gain.

\section{DiffusionGemma SFT Training Data and Evaluation Details}
\label{app:training-evaluation-details}

\paragraph{SWE-bench Verified.}
For the software-engineering comparison in
Figure~\ref{fig:equal-wall-clock-downstream}, the best baseline and CSBP
train on identical
ordered streams of successful CoderForge-Preview trajectories
\citep{CoderForge2026}, formatted for the OpenHands agent
interface~\citep{wang2025openhands}. Trajectories in the training stream have
a mean sequence length of 65{,}536.5 tokens and a maximum length of
131{,}072 tokens.
We evaluate the base model at 0 hours and
checkpoints at 3, 6, 9, and 12 hours, with one attempt on each of the 500
SWE-bench Verified tasks~\citep{jimenez2024swebench,chowdhury2024swebenchverified}.
The harness supplies the issue and repository state; success requires both the
designated \texttt{FAIL\_TO\_PASS} and \texttt{PASS\_TO\_PASS} tests to pass.

\paragraph{Terminal-Bench Lite.}
The best baseline and CSBP train on identical ordered streams of quality-filtered Terminus-2
trajectories from
LiteCoder-Terminal-SFT~\citep{peng2026litecoderterminal}. Trajectories in the
training stream have a mean sequence length of 65{,}551 tokens and a maximum
length of 131{,}072 tokens. We evaluate the same
checkpoint schedule with Harbor's Terminus-2 agent~\citep{Harbor_Framework}
on all 100 OpenThoughts-TBLite tasks~\citep{OpenThoughts-TBLite}, with one
attempt per task. A task passes only when its benchmark verifier returns success.

\paragraph{Matched wall-clock protocol.}
Within each workload, the best baseline and CSBP use the same model, training corpus,
optimizer settings, checkpoint schedule, and evaluation configuration; only
the distributed training method changes. We report pass rate as the percentage
of tasks solved by the single attempt. Evaluation tasks and verifier outputs do
not contribute training gradients. Both sources apply benchmark decontamination:
CoderForge excludes matching SWE-bench Verified repository/commit pairs and
problem statements, while LiteCoder-Terminal-SFT filters 13-gram overlaps with
Terminal-Bench queries~\citep{CoderForge2026,peng2026litecoderterminal}.

\paragraph{SFT optimization.}
The CoderForge-Preview and LiteCoder-Terminal-SFT runs use a global batch size
of one example and one gradient-accumulation step per update.
We train rank-16 LoRA adapters with $\alpha=32$ on eight H100 GPUs. The
baseline uses CP8, while CSBP uses CP8/BP8.
We optimize with DeepSpeed ZeRO-2 and FusedAdam using a peak learning rate of
$10^{-5}$, Adam coefficients $(\beta_1,\beta_2)=(0.95,0.99)$,
$\epsilon=10^{-8}$, weight decay $10^{-4}$, gradient clipping at $1.0$, and a
cosine schedule that decays the learning rate to $10^{-6}$.

\section{End-to-End Numerical Agreement}
\label{app:numerical-agreement}

Table~\ref{tab:objective-equivalence-256k} compares the strongest non-BP
baseline with CSBP using full-depth BF16 runs at 256K context on sixteen H200
GPUs.  The
NemotronDiffusion and DiffusionGemma rows evaluate BDLM fine-tuning, while the
Qwen rows evaluate Fast-dLLM v2 AR-to-BDLM conversion.  Each pair uses the same
model, random seed, generated inputs, objective, and optimizer configuration.
The resulting loss differences are small and consistent with floating-point
reduction order.

\begin{table}[H]
  \centering
  \caption{End-to-end loss agreement between the strongest non-BP baseline and CSBP at 256K context.}
  \label{tab:objective-equivalence-256k}
  \small
  \setlength{\tabcolsep}{5pt}
  \renewcommand{\arraystretch}{1.12}
  \resizebox{\linewidth}{!}{%
  \begin{tabular}{@{}lcllrr@{}}
    \toprule
    Model & Context & \shortstack{Best baseline\\topology} & \shortstack{Ours\\topology} & \shortstack{Absolute loss\\difference} & \shortstack{Relative loss\\difference} \\
    \midrule
    NemotronDiffusion 3B
      & 256K & DP4/TP1/CP4 & DP4/TP1/CP4/BP4
      & $2.83\times10^{-3}$ & $2.11\times10^{-4}$ \\
    NemotronDiffusion 8B
      & 256K & DP4/TP1/CP4 & DP4/TP1/CP4/BP4
      & $2.59\times10^{-3}$ & $1.97\times10^{-4}$ \\
    NemotronDiffusion 14B
      & 256K & DP2/TP1/CP8 & DP2/TP1/CP8/BP8
      & $9.54\times10^{-7}$ & $7.18\times10^{-8}$ \\
    DiffusionGemma 26B-A4B
      & 256K & DP1/TP1/EP2/CP8 & DP1/TP1/EP2/CP8/BP8
      & $6.51\times10^{-3}$ & $2.11\times10^{-4}$ \\
    Qwen3.8-27B
      & 256K & DP1/TP2/CP8 & DP1/TP2/CP8/BP8
      & $7.21\times10^{-6}$ & $5.11\times10^{-7}$ \\
    Qwen3.5-27B
      & 256K & DP1/TP2/CP8 & DP1/TP2/CP8/BP8
      & $9.50\times10^{-5}$ & $6.69\times10^{-6}$ \\
    \bottomrule
  \end{tabular}}
\end{table}

\section{Pure Block Parallelism Ablation}
\label{app:pure-bp-ablation}

Sharding the clean context makes BP practical at long sequence lengths
(Table~\ref{tab:pure-bp-ablation}). For NemotronDiffusion 14B at 64K, CSBP is
$1.46\times$ faster than pure BP and uses 89.8 rather than 110.7\,GiB of HBM.
At 128K, CSBP still fits while pure BP runs out of memory. DiffusionGemma 26B-A4B
shows the same need for context sharding even earlier: pure BP runs out of
memory at both 64K and 128K, where CSBP fits and outperforms pure CP. Block ownership removes corrupted
K/V communication; context sharding prevents replicated clean prefixes from
erasing that benefit.

\begin{table}[H]
  \centering
  \caption{Context sharding enables block parallelism at long contexts.}
  \label{tab:pure-bp-ablation}
  \footnotesize
  \setlength{\tabcolsep}{1.75pt}
  \renewcommand{\arraystretch}{1.08}
  \ifdefined\iclrfinalcopy
    \def\fitpurebptable#1{\resizebox{\linewidth}{!}{#1}}
  \else
    \def\fitpurebptable#1{#1}
  \fi
  \fitpurebptable{%
  \begin{tabular}{lllccrr}
    \toprule
    Model & Context & Method & Topology & Global batch & Tok/s & Peak HBM (GiB) \\
    \midrule
    NemotronDiffusion 14B & 64K & Best baseline & DP8/TP1/CP2 & 8 & 20{,}092 & 89.5 \\
                  &     & Pure BP            & DP8/TP1/BP2 & 8 & 15{,}223 & 110.7 \\
                  &     & Ours               & DP8/TP1/CP2/BP2 & 8 & \textbf{22{,}152} & 89.8 \\
    \addlinespace
    NemotronDiffusion 14B & 128K & Best baseline & DP4/TP1/CP4 & 4 & 13{,}114 & 90.0 \\
                  &      & Pure BP            & DP4/TP1/BP4 & 4 & OOM & OOM \\
                  &      & Ours               & DP4/TP1/CP4/BP4 & 4 & \textbf{15{,}136} & 90.1 \\
    \midrule
    DiffusionGemma 26B-A4B & 64K & Best baseline & DP4/TP1/EP2/CP2 & 8 & 22{,}154 & 109.4 \\
                   &     & Pure BP            & DP4/TP1/EP2/BP2 & 8 & OOM & OOM \\
                   &     & Ours               & DP4/TP1/EP2/CP2/BP2 & 8 & \textbf{27{,}653} & 113.5 \\
    \addlinespace
    DiffusionGemma 26B-A4B & 128K & Best baseline & DP2/TP1/EP2/CP4 & 4 & 14{,}584 & 112.0 \\
                   &      & Pure BP            & DP2/TP1/EP2/BP4 & 4 & OOM & OOM \\
                   &      & Ours               & DP2/TP1/EP2/CP4/BP4 & 4 & \textbf{21{,}049} & 113.6 \\
    \bottomrule
  \end{tabular}}

\end{table}

\section{Load Balancing Analysis}
\label{app:dual-end-scheduling-ablation}

Load balancing delivers substantial gains without increasing memory use
(Table~\ref{tab:dual-end-scheduling-ablation}). Replacing contiguous block
assignment with early--late pairing improves throughput by $1.34\times$ for
NemotronDiffusion 3B and $1.28\times$ for DiffusionGemma 26B-A4B at 256K. Peak
HBM is unchanged for NemotronDiffusion and falls by 6.5\,GiB for DiffusionGemma.
Only the block assignment changes; the objective and
remaining training settings are fixed. These gains show that equal block
counts are not enough: balancing attention work across ranks is important
to realizing CSBP's throughput benefit.

\begin{table}[H]
  \centering
  \caption{Load balancing accelerates CSBP at 256K context.}
  \label{tab:dual-end-scheduling-ablation}
  \small
  \setlength{\tabcolsep}{4pt}
  \renewcommand{\arraystretch}{1.35}
  \resizebox{\linewidth}{!}{%
  \begin{tabular}{@{}llcrrrrr@{}}
    \toprule
    Model & Topology & \shortstack{Global\\batch} & ms/step & Tok/s & \shortstack{MFU\\(\%)} & \shortstack{Peak HBM\\(GiB)} & Speedup \\
    \midrule
    NemotronDiffusion 3B & DP4/TP1/CP4/BP4 & 4 &
    88{,}950.4 / \textbf{66{,}500.3} &
    11{,}788 / \textbf{15{,}768} & 27.8 / \textbf{37.2} &
    57.0 / 57.0 &
    \textbf{1.34$\times$} \\
    \midrule
    DiffusionGemma 26B-A4B & DP1/TP1/EP2/CP8/BP8 & 2 &
    43{,}722.0 / \textbf{34{,}134.6} &
    11{,}991 / \textbf{15{,}359} & 14.4 / \textbf{18.4} &
    120.1 / \textbf{113.6} & \textbf{1.28$\times$} \\
    \bottomrule
  \end{tabular}}
  \par\vspace{2pt}
  \parbox{\linewidth}{\footnotesize\raggedright Pairs: contiguous~/~dual-end. Speedup: dual-end~/~contiguous. Bold marks the better value.}
\end{table}

\section{AR-to-BDLM Conversion Results}
\label{app:fast-dllm-v2-results}

CSBP's benefits extend to converting autoregressive models into BDLMs
(Table~\ref{tab:fast-dllm-v2-context-scaling}). Under the Fast-dLLM v2
objective, both Qwen backbones train faster and use less peak HBM at every
evaluated context length. Speedup grows from $1.26$--$1.27\times$ at 64K to
$1.27$--$1.33\times$ at 256K, where CSBP also saves 9.6\,GiB of HBM.

\begin{table}[H]
  \centering
  \caption{CSBP accelerates AR-to-BDLM conversion from 64K to 256K context.}
  \label{tab:fast-dllm-v2-context-scaling}
  \small
  \setlength{\tabcolsep}{4pt}
  \renewcommand{\arraystretch}{1.35}
  \resizebox{\linewidth}{!}{%
  \begin{tabular}{@{}llllrrrr@{}}
    \toprule
    Model & Context & \shortstack{Best baseline\\topology} & \shortstack{Ours\\topology} & Tok/s & Speedup & \shortstack{MFU\\(\%)} & \shortstack{Peak HBM\\(GiB)} \\
    \midrule
    Qwen3.8-27B & 64K  & DP1/TP2/CP8 & DP1/TP2/CP8/BP8 & 4{,}697 / \textbf{5{,}905} & \textbf{1.26$\times$} & 22.3 / \textbf{28.0} & 64.2 / \textbf{62.5} \\
    Qwen3.8-27B & 128K & DP1/TP2/CP8 & DP1/TP2/CP8/BP8 & 3{,}742 / \textbf{4{,}853} & \textbf{1.30$\times$} & 21.4 / \textbf{27.8} & 83.0 / \textbf{78.3} \\
    Qwen3.8-27B & 256K & DP1/TP2/CP8 & DP1/TP2/CP8/BP8 & 2{,}623 / \textbf{3{,}490} & \textbf{1.33$\times$} & 20.1 / \textbf{26.8} & 120.7 / \textbf{111.1} \\
    \midrule
    Qwen3.5-27B & 64K  & DP1/TP2/CP8 & DP1/TP2/CP8/BP8 & 4{,}709 / \textbf{5{,}967} & \textbf{1.27$\times$} & 22.3 / \textbf{28.3} & 64.2 / \textbf{62.5} \\
    Qwen3.5-27B & 128K & DP1/TP2/CP8 & DP1/TP2/CP8/BP8 & 3{,}722 / \textbf{4{,}850} & \textbf{1.30$\times$} & 21.3 / \textbf{27.8} & 83.0 / \textbf{78.3} \\
    Qwen3.5-27B & 256K & DP1/TP2/CP8 & DP1/TP2/CP8/BP8 & 2{,}641 / \textbf{3{,}351} & \textbf{1.27$\times$} & 20.3 / \textbf{25.7} & 120.7 / \textbf{111.1} \\
    \bottomrule
  \end{tabular}}
  \par\vspace{2pt}
  \parbox{\linewidth}{\footnotesize\raggedright Pairs: best baseline~/~ours. Bold marks the better value.}
\end{table}

\section{Effectiveness Across Different Block Sizes}
\label{app:block-size-sensitivity}

CSBP maintains its advantage across an eightfold range of block sizes
(Table~\ref{tab:block-size-sensitivity}). From 128 to 1,024 tokens per block,
speedup stays within $1.14$--$1.15\times$ for NemotronDiffusion 14B and
$1.33$--$1.44\times$ for DiffusionGemma 26B-A4B.

\begin{table}[H]
  \centering
  \caption{CSBP sustains speedups across block sizes at 128K context.}
  \label{tab:block-size-sensitivity}
  \small
  \setlength{\tabcolsep}{5pt}
  \renewcommand{\arraystretch}{1.35}
  \resizebox{\linewidth}{!}{%
  \begin{tabular}{@{}lrllcrrrrr@{}}
    \toprule
    Model & \shortstack{Block\\size} & \shortstack{Best\\baseline\\topology} & \shortstack{Ours\\topology} & \shortstack{Global\\batch} & ms/step & Tok/s & \shortstack{MFU\\(\%)} & \shortstack{Peak HBM\\(GiB)} & Speedup \\
    \midrule
    NemotronDiffusion 14B & 128 & DP4/TP1/CP4 & DP4/TP1/CP4/BP4 & 4 &
    39{,}187.3 / \textbf{34{,}279.4} & 13{,}379 / \textbf{15{,}295} &
    34.3 / \textbf{39.2} & \textbf{90.0} / 90.1 & \textbf{1.14$\times$} \\
    NemotronDiffusion 14B & 256 & DP4/TP1/CP4 & DP4/TP1/CP4/BP4 & 4 &
    39{,}979.2 / \textbf{34{,}637.4} & 13{,}114 / \textbf{15{,}136} &
    33.7 / \textbf{38.9} & \textbf{90.0} / 90.1 & \textbf{1.15$\times$} \\
    NemotronDiffusion 14B & 512 & DP4/TP1/CP4 & DP4/TP1/CP4/BP4 & 4 &
    38{,}852.1 / \textbf{34{,}197.3} & 13{,}494 / \textbf{15{,}331} &
    34.7 / \textbf{39.4} & \textbf{90.0} / 90.1 & \textbf{1.14$\times$} \\
    NemotronDiffusion 14B & 1{,}024 & DP4/TP1/CP4 & DP4/TP1/CP4/BP4 & 4 &
    39{,}201.0 / \textbf{34{,}392.4} & 13{,}374 / \textbf{15{,}244} &
    34.5 / \textbf{39.3} & \textbf{90.0} / 90.1 & \textbf{1.14$\times$} \\
    \midrule
    DiffusionGemma 26B-A4B & 128 & DP2/TP1/EP2/CP4 & DP2/TP1/EP2/CP4/BP4 & 4 &
    33{,}271.3 / \textbf{24{,}844.4} & 15{,}758 / \textbf{21{,}103} &
    12.6 / \textbf{16.9} & \textbf{111.8} / 113.6 & \textbf{1.34$\times$} \\
    DiffusionGemma 26B-A4B & 256 & DP2/TP1/EP2/CP4 & DP2/TP1/EP2/CP4/BP4 & 4 &
    35{,}949.5 / \textbf{24{,}907.5} & 14{,}584 / \textbf{21{,}049} &
    11.7 / \textbf{16.9} & \textbf{112.0} / 113.6 & \textbf{1.44$\times$} \\
    DiffusionGemma 26B-A4B & 512 & DP2/TP1/EP2/CP4 & DP2/TP1/EP2/CP4/BP4 & 4 &
    33{,}319.8 / \textbf{24{,}829.4} & 15{,}735 / \textbf{21{,}116} &
    12.7 / \textbf{17.0} & \textbf{112.1} / 113.7 & \textbf{1.34$\times$} \\
    DiffusionGemma 26B-A4B & 1{,}024 & DP2/TP1/EP2/CP4 & DP2/TP1/EP2/CP4/BP4 & 4 &
    33{,}033.5 / \textbf{24{,}850.7} & 15{,}871 / \textbf{21{,}097} &
    12.9 / \textbf{17.1} & \textbf{112.3} / 113.7 & \textbf{1.33$\times$} \\
    \bottomrule
  \end{tabular}}
  \par\vspace{2pt}
  \parbox{\linewidth}{\footnotesize\raggedright Pairs: best baseline~/~ours. Bold marks the better value.}
\end{table}

\section{Savings Breakdowns}
\label{app:sources-of-speedup}

Table~\ref{tab:operator-profile-256k} separates attention communication from
end-to-end step time. NemotronDiffusion 14B uses full attention, so keeping
corrupted K/V local halves its logical attention traffic; overlap limits the
reduction in exposed attention communication to 20.4\%. DiffusionGemma 26B-A4B
has a second, architecture-specific source of savings: 25 of its 30 attention
layers use a 1,024-token sliding window. The tested CP8 baseline uses full-shard
K/V exchange in all 30 layers. CSBP keeps corrupted K/V local and, in the 25
sliding-window layers, its window-aware exchange sends only the clean rows
needed by local queries; only the five global-attention layers exchange
complete clean K/V shards. This reduces logical attention traffic by 93.5\%
and exposed attention communication by 93.4\%.

\begin{table}[H]
  \centering
  \caption{Attention and step-time breakdown at 256K: CP8 baseline vs. CSBP.}
  \label{tab:operator-profile-256k}
  \small
  \setlength{\tabcolsep}{3pt}
  \renewcommand{\arraystretch}{1.12}
  \begin{tabular}{lcc}
    \toprule
    Measure & \shortstack{NemotronDiffusion 14B\\CP8 $\rightarrow$ CP8/BP8} & \shortstack{DiffusionGemma 26B-A4B\\CP8 $\rightarrow$ CP8/BP8} \\
    \midrule
    Step time (s)
      & 66.50 $\rightarrow$ 56.37 \textcolor{green!50!black}{\textbf{($1.18\times$)}}
      & 49.66 $\rightarrow$ 34.13 \textcolor{green!50!black}{\textbf{($1.45\times$)}} \\
    Forward GPU span (s)
      & 17.57 $\rightarrow$ 13.00 \textcolor{green!50!black}{\textbf{($-26.0\%$)}}
      & 10.99 $\rightarrow$ 5.83 \textcolor{green!50!black}{\textbf{($-47.0\%$)}} \\
    Backward GPU span (s)
      & 49.46 $\rightarrow$ 45.02 \textcolor{green!50!black}{\textbf{($-9.0\%$)}}
      & 36.50 $\rightarrow$ 26.62 \textcolor{green!50!black}{\textbf{($-27.1\%$)}} \\
    Attention traffic (GiB/GPU)
      & 210.00 $\rightarrow$ 105.00 \textcolor{green!50!black}{\textbf{($-50.0\%$)}}
      & 433.13 $\rightarrow$ 28.30 \textcolor{green!50!black}{\textbf{($-93.5\%$)}} \\
    Exposed attention comm. (s/GPU)
      & 3.034 $\rightarrow$ 2.414 \textcolor{green!50!black}{\textbf{($-20.4\%$)}}
      & 10.939 $\rightarrow$ 0.723 \textcolor{green!50!black}{\textbf{($-93.4\%$)}} \\
    \bottomrule
  \end{tabular}
\end{table}

The 93\% reductions apply to attention communication, not the complete
training step. DiffusionGemma's step time falls by 31.3\%, while dense and
expert GEMM durations remain nearly unchanged. Both forward and backward
become faster, yielding $1.18\times$ and $1.45\times$ end-to-end speedup for
NemotronDiffusion 14B and DiffusionGemma 26B-A4B, respectively.

\section{Complete H200 Profiles}
\label{app:h200-profiles}

CSBP improves the throughput--memory trade-off across the tested models and
context lengths (Table~\ref{tab:h200-two-node-1p2}). The independently selected
configurations deliver speedups up to $1.61\times$. Even when throughput is
nearly tied, context sharding can substantially reduce memory: for NemotronDiffusion
3B at 128K, CSBP stays within 1\% of the DP16 baseline's throughput while
saving 41\,GiB of peak HBM. This case highlights another practical benefit
of CSBP: retaining throughput while freeing memory.

\begin{table}[H]
  \centering
  \caption{Throughput and memory gains across long-context training configurations.}
  \label{tab:h200-two-node-1p2}
  \small
  \setlength{\tabcolsep}{5pt}
  \renewcommand{\arraystretch}{1.15}
  \resizebox{\linewidth}{!}{%
  \begin{tabular}{@{}llllcrrrrr@{}}
    \toprule
    Model & Context & \shortstack{Ours\\topology} & \shortstack{Best\\baseline\\topology} & \shortstack{Global\\batch} & Tok/s & \shortstack{MFU\\(\%)} & \shortstack{Peak HBM\\(GiB)} & \shortstack{HBM saved\\(GiB)} & Speedup \\
    \midrule
    NemotronDiffusion 3B & 128K & DP8/TP1/CP2/BP2 & DP16/TP1 & 8 / 16 & 28{,}891 / 29{,}086 & 37.5 / 37.7 & 56.5 / 97.5 & 41.0 & 0.99$\times$ \\
    NemotronDiffusion 3B & 256K & DP4/TP1/CP4/BP4 & DP4/TP1/CP4 & 4 & 16{,}027 / 13{,}434 & 37.8 / 31.7 & 57.0 / 57.1 & 0.1 & \textbf{1.19$\times$} \\
    NemotronDiffusion 3B & 512K & DP2/TP1/CP8/BP8 & DP2/TP1/CP8 & 2 & 8{,}352 / 6{,}895 & 37.4 / 30.8 & 58.1 / 61.7 & 3.6 & \textbf{1.21$\times$} \\
    \midrule
    NemotronDiffusion 8B & 64K  & DP8/TP1/CP2/BP2 & DP8/TP1/CP2 & 16 & 29{,}829 / 26{,}669 & 37.8 / 33.8 & 96.9 / 96.4 & $-0.5$ & \textbf{1.12$\times$} \\
    NemotronDiffusion 8B & 128K & DP8/TP1/CP2/BP2 & DP8/TP1/CP2 & 8  & 19{,}653 / 16{,}580 & 38.5 / 32.5 & 96.9 / 96.4 & $-0.5$ & \textbf{1.19$\times$} \\
    NemotronDiffusion 8B & 256K & DP4/TP1/CP4/BP4 & DP4/TP1/CP4 & 4  & 11{,}379 / 9{,}668 & 38.0 / 32.3 & 97.5 / 97.5 & 0.0 & \textbf{1.18$\times$} \\
    NemotronDiffusion 8B & 512K & DP2/TP1/CP8/BP8 & DP2/TP1/CP8 & 2  & 6{,}254 / 5{,}202 & 38.2 / 31.8 & 98.5 / 99.5 & 1.0 & \textbf{1.20$\times$} \\
    \midrule
    NemotronDiffusion 14B & 64K  & DP8/TP1/CP2/BP2 & DP8/TP1/CP2 & 8 & 22{,}152 / 20{,}092 & 38.8 / 35.2 & 89.8 / 89.5 & $-0.3$ & \textbf{1.10$\times$} \\
    NemotronDiffusion 14B & 128K & DP4/TP1/CP4/BP4 & DP4/TP1/CP4 & 4 & 15{,}136 / 13{,}114 & 38.9 / 33.7 & 90.1 / 90.0 & $-0.1$ & \textbf{1.15$\times$} \\
    NemotronDiffusion 14B & 256K & DP2/TP1/CP8/BP8 & DP2/TP1/CP8 & 2 & 9{,}301 / 7{,}884 & 39.0 / 33.1 & 90.6 / 91.1 & 0.5 & \textbf{1.18$\times$} \\
    NemotronDiffusion 14B & 512K & DP1/TP2/CP8/BP8 & DP1/TP2/CP8 & 1 & 5{,}052 / 4{,}242 & 37.7 / 31.6 & 71.9 / 72.7 & 0.8 & \textbf{1.19$\times$} \\
    \midrule
    DiffusionGemma 26B-A4B & 64K  & DP4/TP1/EP2/CP2/BP2 & DP4/TP1/EP2/CP2 & 8 & 27{,}653 / 22{,}154 & 16.0 / 12.8 & 113.5 / 109.4 & $-4.1$ & \textbf{1.25$\times$} \\
    DiffusionGemma 26B-A4B & 128K & DP2/TP1/EP2/CP4/BP4 & DP2/TP1/EP2/CP4 & 4 & 21{,}049 / 14{,}584 & 16.9 / 11.7 & 113.6 / 112.0 & $-1.6$ & \textbf{1.44$\times$} \\
    DiffusionGemma 26B-A4B & 256K & DP1/TP1/EP2/CP8/BP8 & DP1/TP1/EP2/CP8 & 2 & 15{,}359 / 10{,}557 & 18.4 / 12.7 & 113.6 / 118.1 & 4.6 & \textbf{1.45$\times$} \\
    DiffusionGemma 26B-A4B & 512K & DP1/TP2/CP8/BP8/SP & DP1/TP2/CP8/SP & 1 & 9{,}815 / 6{,}079 & 20.3 / 12.6 & 106.5 / 127.4 & 20.9 & \textbf{1.61$\times$} \\
    \bottomrule
  \end{tabular}}
  \par\vspace{2pt}
  \parbox{\linewidth}{\footnotesize\raggedright Pairs: ours~/~best baseline. HBM saved: best baseline minus ours.}
\end{table}

\end{document}